\documentclass{article}

\usepackage{graphicx}
\usepackage{hyperref}
\usepackage{url}
\usepackage{bm,amsmath,amsfonts}
\usepackage[inline]{enumitem}
\usepackage{subcaption}
\usepackage{mathtools}
\usepackage{chngpage}
\usepackage[utf8]{inputenc}
\usepackage{booktabs}
\usepackage{nicefrac}
\usepackage{microtype}
\usepackage{color}
\usepackage{multirow}
\usepackage{comment}
\usepackage{float}
\usepackage{bm,amsmath,amsfonts,amssymb}
\usepackage[inline]{enumitem}
\usepackage{subcaption}
\usepackage{mathtools}
\usepackage[dvipsnames]{xcolor}
\usepackage{tcolorbox}
\usepackage[hang,flushmargin]{footmisc}

\usepackage[accepted]{icml2020}

\def\Re{\mathbb{R}}

\def\Nat{{\rm I\kern\pIR N}}
\def\argmax{\mathop{\rm arg\,max}}

\newcommand{\EE}[1]{\exptE\left[#1\right]}

\newcommand{\defeq}{\overset{\text{\tiny def}}{=}}

\newcommand{\eye}{\mathbf{I}}

\def\E{{\mathcal{E}}}

\def\S{{\mathcal{S}}}

\DeclareMathOperator{\hvec}{\bf h}
\DeclareMathOperator{\xvec}{\bf x}
\DeclareMathOperator{\wvec}{\bf w}
\DeclareMathOperator{\zvec}{\bf z}
\DeclareMathOperator{\bvec}{\bf b}
\DeclareMathOperator{\gvec}{\bf g}

\DeclareMathOperator{\yvec}{\bf y}
\DeclareMathOperator{\Ivec}{\bf I}
\DeclareMathOperator{\Avec}{\bf A}
\DeclareMathOperator{\Gvec}{\bf G}

\DeclareMathOperator{\Hvec}{\bf H}
\DeclareMathOperator{\Cvec}{\bf C}

\DeclareMathOperator{\Vvec}{\bf V}
\DeclareMathOperator{\Xvec}{\bf X}
\DeclareMathOperator{\varrhovec}{\boldsymbol{\varrho}}

\def\vectheta{{\boldsymbol{\theta}}}

\def\vec0{\mathbf 0}
\def\vecb{\mathbf b}

\def\vech{\mathbf h}

\def\vecw{\mathbf w}
\def\vecx{\mathbf x}

\newtheorem{thm}{Theorem}[section]

\newcommand{\beq}{\begin{equation}}
\newcommand{\eeq}{\end{equation}}
\newcommand{\beqa}{\begin{eqnarray}}
\newcommand{\eeqa}{\end{eqnarray}}
\newcommand{\beqan}{\begin{eqnarray*}}
\newcommand{\eeqan}{\end{eqnarray*}}
\newcommand{\ben}{\begin{eqnarray*}}
\newcommand{\een}{\end{eqnarray*}}

\def\tr{^\top\!}

\newcommand{\Amat}{\mathbf{A}}
\newcommand{\Cmat}{\mathbf{C}}
\newcommand{\inv}{{-1}}
\newcommand{\sneg}{\mathrm{-}}

\renewcommand{\EE}[2]{\mathbb{E}_{#1\!\!}\left[#2\right]}

\newcommand{\CEE}[3]{\EE{#1}{{#2}~\middle\vert~{#3}}}
\renewcommand{\CEE}[3]{\EE{#1}{{#2}\mid{#3}}}
\def\CE#1#2{\CEE{\,}{#1}{#2}}

\def\E#1{\EE{\,}{#1}}
\def\Epi#1{\EE{\pi}{#1}}

\newcommand{\zerovec}{\mathbf{0}}

\def\gradSA#1#2{\nabla_{\vecw}{\hat{q}({#1, #2})}}

\newcommand\blfootnote[1]{%
  \begingroup
  \renewcommand\thefootnote{}\footnote{#1}%
  \addtocounter{footnote}{-1}%
  \endgroup
}

\icmltitlerunning{Gradient Temporal-Difference Learning with Regularized Corrections}

\begin{document}

\twocolumn[
  \icmltitle{Gradient Temporal-Difference Learning with Regularized Corrections}



\icmlsetsymbol{equal}{*}

\begin{icmlauthorlist}
\icmlauthor{Sina Ghiassian}{equal,UofA}
\icmlauthor{Andrew Patterson}{equal,UofA}
\icmlauthor{Shivam Garg}{UofA}
\icmlauthor{Dhawal Gupta}{UofA}
\icmlauthor{Adam White}{UofA,Deepmind}
\icmlauthor{Martha White}{UofA}
\end{icmlauthorlist}

\icmlaffiliation{UofA}{Amii, Department of Computing Science, University of Alberta.}
\icmlaffiliation{Deepmind}{DeepMind, Alberta}

\icmlcorrespondingauthor{Sina Ghiassian}{ghiassia@ualberta.ca}
\icmlcorrespondingauthor{Andrew Patterson}{ap3@ualberta.ca}

\icmlkeywords{Machine Learning, Reinforcement Learning, Off-policy Learning, ICML}

\vskip 0.3in
]

\printAffiliationsAndNotice{\icmlEqualContribution}
\begin{abstract}
It is still common to use Q-learning and temporal difference (TD) learning---even though they have divergence issues and sound Gradient TD alternatives exist---because divergence seems rare and they typically perform well.
However, recent work with large neural network learning systems reveals that instability is more common than previously thought.
Practitioners face a difficult dilemma: choose an easy to use and performant TD method, or a more complex algorithm that is more sound but harder to tune and all but unexplored with non-linear function approximation or control.
In this paper, we introduce a new method called TD with Regularized Corrections (TDRC), that attempts to balance ease of use, soundness, and performance.
It behaves as well as TD, when TD performs well, but is sound in cases where TD diverges.
We empirically investigate TDRC across a range of problems, for both prediction and control, and for both linear and non-linear function approximation, and show, potentially for the first time, that Gradient TD methods could be a better alternative to TD and Q-learning.
\end{abstract}

\section{Introduction}
\label{sct:Intro}
Off-policy learning---the ability to learn the policy or value function for one policy while following another---underlies many practical implementations of reinforcement learning. Many systems use experience replay, where the value function is updated using previous experiences under many different policies. A similar strategy is employed in asynchronous learning systems that use experience from several different policies to update multiple distributed learners (Espeholt et al., 2018). Off-policy updates can also be used to learn a policy from human demonstrations. In general, many algorithms attempt to estimate the optimal policy from samples generated from a different exploration policy. One of the most widely-used algorithms, Q-learning---a temporal difference (TD) algorithm---is off-policy by design: simply updating toward the maximum value action in the current state, regardless of which action the agent selected.

Both TD and Q-learning, however, have well documented convergence issues, as highlighted in the seminal counterexample by Baird (1995).
The fundamental issue is the combination of function approximation, off-policy updates, and bootstrapping: an algorithmic strategy common to sample-based TD learning and Dynamic Programming algorithms (Precup, Sutton \& Dasgupta, 2001). This combination can cause the value estimates to grow without bound (Sutton \& Barto, 2018).
Baird's result motivated over a decade of research and several new off-policy algorithms. The most well-known of these approaches, the Gradient TD methods (Sutton et al., 2009), make use of a second set of weights and importance sampling.

Although sound under function approximation, these Gradient TD methods are not commonly used in practice, likely due to the additional complexity of tuning two learning rate parameters.
Many practitioners continue to use unsound approaches such as TD and Q-learning for good reasons. The evidence of divergence is based on highly contrived toy counter-examples. Often, many large scale off-policy learning systems are designed to ensure that the target and behaviour policies are similar---and therefore less off-policy---by ensuring prioritization is mixed with random sampling (Schaul et al., 2016), or frequently syncing the actor policies in asynchronous architectures (Mnih et al., 2016).
However, if agents could learn from a larger variety of data streams, our systems could be more flexible and potentially more data efficient. Unfortunately, it appears that current architectures are not as robust under these more aggressive off-policy settings (van Hasselt et al., 2018).
This results in a dilemma: the easy-to-use and typically effective TD algorithm can sometimes fail, but the sound Gradient TD algorithms can be difficult to use.

There are algorithms that come close to achieving convergence and lower variance updates without the need to tune multiple stepsize parameters. Retrace (Munos et al., 2016) and its prediction variant Vtrace (Espeholt et al., 2018) reduce the variance of off-policy updating, by clipping importance sampling ratios. These methods, however, are built on off-policy TD and so still have divergence issues (Touati et al., 2018). The sound variants of these algorithms (Touati et al., 2018), and the related work on an algorithm called ABQ (Mahmood, Yu \& Sutton, 2017), maintain some of the variance reduction, but rely on Gradient TD to obtain soundness and so inherit the issues therein---the need to tune multiple stepsize parameters. Linear off-policy prediction can be reformulated as a saddlepoint problem, resulting in one time-scale, true gradient descent variant of the GTD2 algorithm (Mahadevan et al., 2014; Liu et al., 2015; Liu et al., 2016).
The Emphatic TD algorithm achieves convergence with linear function approximation and off-policy updates using only a single set of weights and thus one stepsize parameter (Sutton et al., 2016). Unfortunately, high variance updates reduce the practicality of the method (White \& White, 2016). Finally, Hybrid TD algorithms (Hackman, 2012, White \& White, 2016) were introduced to automatically switch between TD updates when the data is on-policy, and gradient-style updates otherwise, thus ensuring convergence. In practice these hybrid methods are more complicated to implement and can have stability issues (White \& White, 2016).

In this paper we introduce a new Gradient TD method, called TD with Regularized Corrections (TDRC).
With more regularization, the algorithm acts like TD, and with no regularization, it reduces to TD with gradient Corrections (TDC).
We find that for an interim level of regularization, TDRC obtains the best of both algorithms, and is not sensitive to this parameter: a regularization parameter of 1.0 was effective across all experiments.
We show that our method (1) outperforms other Gradient TD methods overall across a variety of problems, and (2) matches TD when TD performs well while maintaining convergence guarantees.
We demonstrate that TDC frequently outperforms the saddlepoint variant of Gradient TD, motivating why we build on TDC and the utility of being able to shift between TD and TDC by setting the regularization parameter.
We then highlight why TDRC improves so significantly on TDC, by examining TDC's sensitivity to its second stepsize. We conclude with a demonstration in control, with non-linear function approximation, showing that (1) TDC can perform very well in some settings and very poorly in others, and (2) TDRC is always comparable to Q-learning, and in some cases, is much better.

\section{Background}
\label{sct:TDLearningWithLinearValueFunctionApproximation}

In this paper we tackle the policy evaluation problem in Reinforcement Learning. We model the agent's interactions with its environment as a Markov Decision Process (MDP). The agent and environment interact continually.
On each time step $t=0, 1, 2, \ldots,$
the agent selects an action $A_t \in \mathcal{A}$ in state $S_t \in \S$. Based on the agent's action $A_t$ and the transition dynamics, $P: \mathcal{S} \times \mathcal{A} \times \mathcal{S} \rightarrow [0,1]$, the environment transitions into a new state, $S_{t+1}$, and emits a scalar reward $R_{t+1}$. The agent selects actions according to its policy $\pi:  \mathcal{S} \times \mathcal{A} \rightarrow [0,1]$.
The main objective in policy evaluation is to estimate the value of a state $s$, defined as the expected discounted sum of future rewards under $\pi$:
\begin{align}v_\pi(s) &\defeq \Epi{R_{t+1} + \gamma R_{t+2} + \gamma^2 R_{t+3} + \cdots | S_t = s}\nonumber\\ &= \Epi{G_t | S_t = s}
  \label{eq:value_func},
\end{align}
where $\gamma\in [0,1]$, $G_t\in\mathbb{R}$ is called the {\em return}, and $\mathbb{E}_{\pi}$ is the expectation taken with respect to future states, actions, and rewards generated by $\pi$ and $P$.

In many problems of interest, the agent cannot directly observe the state.
Instead, on each step the agent observes a featurized representation of the state $\vecx_t \defeq \vecx(S_t) \in \Re^n$, where $n \ll |\mathcal{S}|$.
In this setting, the agent cannot estimate the value of each state individually, but must approximate the value with a parametric function.
In this paper, we focus on the case of linear function approximation, where the value estimate $\hat v: \mathcal{S} \times \Re^n \rightarrow \Re$ is simply formed as an inner product between $\vecx(s)$ and a learned set of weights $\vecw \in \Re^n$ given by $\hat v(s, \vecw) \defeq \vecw\tr\vecx(s).$

Our objective is to adjust $\vecw_t$ on each time step to construct a good approximation of the true value: $\hat v \approx v_\pi$. Perhaps the most well known and successful algorithm for doing so is temporal difference (TD) learning :
\begin{align}\label{td_alg}
  &\delta_t \defeq R_{t+1} + \gamma \vecw_t\tr\vecx_{t+1} - \vecw_t\tr\vecx_t \nonumber\\
  &\vecw_{t+1} \leftarrow \vecw_t + \alpha_t \delta_t \vecx_t
\end{align}
for stepsize $\alpha_t > 0$.
TD is guaranteed to be convergent under linear function approximation and on-policy sampling.

The classical TD algorithm was designed for on-policy learning; however, it can be easily extended to the off-policy setting.
In {\em on-policy} learning, the policy used to select actions is the same as the policy used to condition the expectation in the definition of the value function (Eq. \ref{eq:value_func}).
Alternatively, we might want to make {\em off-policy} updates, where the actions are chosen according to some {\em behavior policy} $b$, different from the {\em target policy} $\pi$ used in Eq. \ref{eq:value_func}.
If we view value estimation as estimating the expected return, this off-policy setting corresponds to estimating an expectation conditioned on one distribution with samples collected under another.
TD can be extended to make off-policy updates by using importance sampling ratios $\rho_t\defeq \frac{\pi(A_t|S_t)}{b(A_t|S_t)} \ge 0$.
The resulting algorithm is a minor modification of TD, $\vecw_{t+1} \leftarrow \vecw_t + \alpha_t \rho_t \delta_t \vecx_t,$
where $\delta_t$ is defined in Eq. \ref{td_alg}.

Off-policy TD can diverge with function approximation, but fortunately there are several TD-based algorithms that are convergent.
When TD learning converges, it converges to the TD fixed point: the weight vector where $\mathbb{E}[\delta_t \vecx_t] = 0$.
Interestingly, TD does not perform gradient descent on any objective to reach the TD fixed point. So, one way to achieve convergence is to perform gradient descent on an objective whose minimum corresponds to the TD-fixed point.
Gradient TD methods do exactly this on the Mean Squared Projected Bellman Error (MSPBE) (see Eq. \ref{eq:MSPBEInExpectation}).

There are several ways to approximate and simplify the gradient of MSPBE, each resulting in a different algorithm.
The two most well-known approaches are TD with Corrections (TDC) and Gradient TD (GTD2).
Both these require double the computation and storage of TD, and employ a second set of learned weights $\vech \in\Re^n$ with a different stepsize parameter $\eta \alpha_t$, where $\eta$ is a tunable constant.
The updates for the TDC algorithm otherwise are similar to TD:
\begin{align}\label{eq:GTD}
  \vecw_{t+1} \leftarrow& ~\vecw_t + \alpha_t \rho_t \delta_t \vecx_t - \alpha_t \rho_t \gamma (\vech_t\tr \vecx_t)\vecx_{t+1} \nonumber\\
  \vech_{t+1} \leftarrow& ~\vech_t + \eta\alpha_t\bigl[\rho_t \delta_t - (\vech_t\tr \vecx_t) \bigr] \vecx_t
  .
\end{align}
The GTD2 algorithm uses the same update for $\vech_t$, but the update to the primary weights is different:
\begin{equation}
  \vecw_{t+1} \leftarrow ~\vecw_t + \alpha_t \rho_t (\vecx_t - \gamma \vecx_{t+1}) (\vech_t^{\tr}\vecx_t) \label{eq:GTD2Update}
  .
\end{equation}
The Gradient TD algorithms are not widely used in practice and are considered difficult to use.
In particular, for TDC, the second stepsize has a big impact on performance (White \& White, 2016), and the theory suggests that $\eta>1$ is necessary to guarantee convergence (Sutton et al., 2009).

Attempts to improve Gradient TD methods has largely come from rederiving GTD2 using a saddlepoint formulation of the MSPBE (Mahadevan et al., 2014).
This formulation enables us to view GTD2 as a one-time scale algorithm with a single set of weights $[ \vecw, \vech]$ using a single global stepsize parameter.
In addition, saddlepoint GTD2 can be combined with acceleration techniques like Mirror Prox (Mahadevan et al., 2014) and stochastic variance reduction methods such as SAGA and SVRG (Du et al., 2017). Unfortunately, Mirror Prox has never been shown to improve performance over vanilla GTD2 (White \& White, 2016; Ghiassian et al., 2018). Current variance reduction methods like SAGA are only applicable in the offline setting, and extension to the online setting would require new methods (Du et al., 2017). In Appendix \ref{app_acc} we include comparisons of off-policy prediction algorithms in the batch setting, including recent Kernel Residual Gradient methods (Feng et al., 2019). These experiments suggest that accelerations do not change the relative ranking of the algorithms in the batch setting.

TD is widely considered more sample efficient than all the methods discussed above. A less well-known family of algorithms, called Hybrid methods (Maei, 2011; Hackman, 2012; White \& White, 2016), were designed to exploit the sample efficiency of TD when data is generated on-policy---they reduce to TD in the on-policy setting---and use gradient corrections, like TDC, when the data is off-policy.
These methods provide some of the ease-of-use benefits of TD, but unfortunately do not enjoy the same level of stability as the Gradient TD methods: for instance, HTD can diverge on Baird's counterexample (White \& White, 2016).

\section{TD with Regularized Corrections}
\label{sct:RegularizedTDC}

In this section we develop a new algorithm, called TD with Regularized Corrections (TDRC).
The idea is very simple: to regularize the update to the secondary parameters $\vech$.
The inspiration for the algorithm comes from behavior observed in experiments (see Section \ref{sct:ExperimentalResults}).
Consistently, we find that TDC outperforms---or is comparable to---GTD2 in terms of optimizing the MSPBE; as we reaffirm in our experiments.
These results match previous experiments comparing these two algorithms (White \& White, 2016; Ghiassian et al., 2018).
Previous results suggested that TDC could match TD (White \& White, 2016); but, as we highlight in Section \ref{sct:ExperimentalResults}, this is only when the second stepsize is set so small that TDC is effectively behaving like TD.
This behavior is unsatisfactory because to have guaranteed convergence---e.g. on Baird's Counterexample---the second stepsize needs to be large.
Further, it is somewhat surprising that attempting to obtain an estimate of the gradient of the MSPBE, as done by TDC, can perform so much more poorly than TD.

Notice that the $\vech$ update is simply a linear regression update for estimating the (changing) target $\delta_t$ conditioned on $\vecx_t$, for both GTD2 and TDC.
As $\vecw$ converges, $\delta_t$ approaches zero, and consequently $\vech$ goes to $\zerovec$ as well.
But, a linear regression estimate of $\mathbb{E}[\delta_t | S_t = s]$ is not necessarily the best choice. In fact, using ridge regression---$\ell_2$ regularization---can provide a better bias-variance trade-off: it can significantly reduce variance without incurring too much bias. This is in particular true for $\vech$, where asymptotically $\vech = 0$ and so the bias disappears.

This highlights a potential reason that TD frequently outperforms TDC and GTD2 in experiments: the variance of $\vech$.
If TD already performs well, it is better to simply use the zero variance but biased estimate $\vech_t = \zerovec$.
Adding $\ell_2$ regularization with parameter $\beta$, i.e. $\beta\| \vech \|_2^2$, provides a way to move between TD and TDC.
For a very large $\beta$, $\vech$ will be pushed close to zero and the update to $\vecw$ will be lower variance and more similar to the TD update. On the other hand, for $\beta = 0$, the update reduces to TDC and the estimator $\vech$ will be an unbiased estimator with higher variance.

The resulting update equations for TDRC are
\begin{align}
  \vecw_{t+1} \leftarrow& ~\vecw_t + \alpha \rho_t \delta_t \vecx_t - \alpha \rho_t \gamma (\vech_t^{\tr}\vecx)\vecx_{t+1} \label{eq:regh_secondary}\\
  \vech_{t+1} \leftarrow& ~\vech_t + \alpha \bigl[\rho_t\delta_t - (\vech_t^{\tr} \vecx_t) \bigr] \vecx_t - \alpha \beta \vech_t. \label{eq:regh_primary}
\end{align}
The update to $\vecw$ is the same as TDC, but the update to $\vech$ now has the additional term $\alpha \beta \vech_t$ which corresponds to the gradient of the $\ell_2$ regularizer.
The updates only have a single shared stepsize, $\alpha$, rather than a separate stepsize for the secondary weights $\vech$.
We make this choice precisely for our motivated reason upfront: for ease-of-use.
Further, we find empirically that this choice is effective, and that the reasons for TDC's sensitivity to the second stepsize are mainly due to the fact that a small second stepsize enables TDC to behave like TD (see Section \ref{sec_overall}).
Because TDRC has this behavior by design, a shared stepsize is more effective.

While there are many approaches to reduce the variance of the estimator, $\vech$, we use an $\ell_2$ regularizer because (1) using the $\ell_2$ regularizer ensures the set of solutions for TDRC match TD; (2) the resulting update is asymptotically unbiased, because it biases towards the known asymptotic solution of $\vech$; and (3) the strongly convex $\ell_2$ regularizer improves the convergence rate.
TDC convergence proofs impose conditions on the size of the stepsize for $\vech$ to ensure that it converges more quickly than the ``slow-learner'' $\vecw$, and so increasing convergence rate for $\vech$ should make it easier to satisfy this condition.
Additionally, the $\ell_2$ regularizer biases the estimator $\vech$ towards $\vech = \zerovec$, the known optimum of the learning system as $\vecw$ converges.
This means that the bias imposed on $\vech$ disappears asymptotically, changing only the transient trajectory (we prove this in Theorem~\ref{thm:convergence_tdrc}).

As a final remark, we motivate that TDRC should not require a second stepsize, but have introduced a new parameter ($\beta$) to obtain this property. The idea, however, is that TDRC should be relatively insensitive to $\beta$. The choice of $\beta$ sweeps between two reasonable algorithms: TD and TDC. If we are already comfortable using TD, then it should be acceptable to use TDRC with a larger $\beta$. A smaller $\beta$ will still result in a sound algorithm, though its performance may suffer due to the variance of the updates in $\vech$. In our experiments, we in fact find that TDRC performs well for a wide range of $\beta$, and that our default choice of $\beta = 1.0$ works reasonably across all the problems that we tested.

\subsection{Theoretically Characterizing the TDRC Update}
\label{sec: expected_updates}

The MSPBE (Sutton et al., 2009) is defined as
\begin{align}
  \text{MSPBE}(\vecw_t) &\defeq \E{\delta_t \vecx_t}\tr \E{\vecx_t \vecx_t\tr}^{-1} \E{\delta_t \vecx_t} \label{eq:MSPBEInExpectation}\\
  &= (\sneg\Amat \vecw + \vecb)\tr \Cmat^\inv (\sneg\Amat \vecw + \vecb) \nonumber
\end{align}
where $\E{\delta_t \vecx_t} = \vecb - \Amat \vecw_t$ for
\begin{align*}
  \Cmat &\defeq \E{\vecx \vecx\tr}, \ \ \
  \Amat \defeq \E{\vecx (\vecx - \gamma \vecx' )\tr}, \ \ \
  \vecb \defeq \E{R \vecx}
  .
\end{align*}
The TD fixed point corresponds to $\E{\delta_t \vecx_t} = \zerovec$ and so to the solution to the system $\Amat \vecw_t = \vecb$. The expectation is taken with respect to the target policy $\pi$, unless stated otherwise.

The expected update for TD corresponds to $\E{\delta_t \vecx_t} = \vecb - \Amat \vecw_t$. The expected update for $\vecw$ in TDC corresponds to the gradient of the MSPBE,
\begin{equation*}
  -\frac{1}{2}\nabla \text{MSPBE}(\vecw_t) = \Amat^\top \Cmat^\inv (\vecb - \Amat \vecw_t)
  .
\end{equation*}
Both TDC and GTD2 estimate $\vech \defeq \Cmat^\inv (\vecb - \Amat \vecw_t) =  \E{\vecx_t \vecx_t\tr}^{-1} \E{\delta_t \vecx_t}$, to get the least squares estimate $\vech\tr \vecx_t \approx  \E{\delta_t | \vecx_t}$ for targets $\delta_t$. TDC rearranges terms, to sample this gradient differently than GTD2; for a given $\vech$, both have the same expected update for $\vecw$: $\Amat^\top \vech$.

We can now consider the expected update for TDRC. Solving for the $\ell_2$ regularized problem with target $\delta_t$, we get $(\E{\vecx_t \vecx_t\tr} + \beta \eye) \vech = \E{\delta_t \vecx_t}$ which implies $\vech_\beta = \Cmat_\beta^\inv (\vecb - \Amat \vecw_t)$ for $\Cmat_\beta \defeq \Cmat + \beta \eye$. To get a similar form to TDC, we consider the modified expected update $\Amat_\beta\tr \vech_\beta$ for $\Amat_\beta \defeq \Amat + \beta \eye$. We can get the TDRC update by rearranging this expected update, similarly to how TDC is derived
\begin{align*}
  &\Amat_\beta^\top \vech_\beta = (\E{(\vecx - \gamma \vecx' ) \vecx\tr} + \beta \eye) \vech_\beta\\
  &= \left(\E{\vecx\vecx\tr} + \beta \eye - \gamma \E{\vecx' \vecx\tr} \right) \Cmat_\beta^\inv \E{\delta_t \vecx_t} \\
  &= \left(\E{\vecx\vecx\tr} + \beta \eye \right) \Cmat_\beta^\inv\E{\delta_t \vecx_t} -\gamma \E{\vecx' \vecx\tr}\Cmat_\beta^\inv \E{\delta_t \vecx_t} \\
  &= \E{\delta_t \vecx_t}  - \gamma \E{\vecx' \vecx\tr}\vech_\beta
\end{align*}
This update equation for the primary weights looks precisely like the update in TDC, except that our $\vech$ is estimated differently. Despite this difference, we show in Theorem \ref{thm:fixed_point} (in Appendix \ref{sec: proof_fixed_point}) that the set of TDRC solutions $\vecw$ to $\Amat_\beta^\top \vech_\beta = \zerovec$ includes the TD fixed point, and this set is exactly equivalent if $\Amat_\beta$ is full rank.
%And do we still maintain the key convergence property for the update, namely that we have a positive semi-definite matrix $\Amat_\beta^\top \Cmat_\beta^\inv \Amat$ for our iterative update $\Amat_\beta^\top \Cmat_\beta^\inv (\vecb -\Amat \vecw)$?

\newcommand{\eproof}{$\null\hfill\blacksquare$}
\newenvironment{proof}{\par\noindent{\bf Proof:\ }}{\eproof} %{\hfill\BlackBox\\[.3mm]}

In the following theorem (proof in Appendix \ref{sec: proof_main_thm}) we directly compare convergence of TDRC to TDC. Though the TDRC updates are no longer gradients, we maintain the convergence properties of TDC. This theorem extends the TDC convergence result to allow for $\beta > 0$, where TDC corresponds to TDRC with $\beta = 0$.
\begin{thm}[Convergence of TDRC] \label{thm:convergence_tdrc}
  Consider the TDRC update, with a TDC like stepsize multiplier $\eta \ge 0$:
  \begin{align}
    \hvec_{t+1} &= \hvec_t + \eta \alpha_t \left[\rho_t \delta_t - \hvec_t^\top \xvec_t \right] \xvec_t - \eta \alpha_t \beta \hvec_t, \label{eq: tdrc_h} \\
    \wvec_{t+1} &= \wvec_t + \alpha_t \rho_t \delta_t \xvec_t - \alpha_t \rho_t \gamma (\hvec_t^\top \xvec_t) \xvec_{t+1}, \label{eq: tdrc_w}
  \end{align}
  with stepsizes $\alpha_t \in (0, 1]$, satisfying $\sum_{t=0}^\infty \alpha_t = \infty$ and $\sum_{t=0}^\infty \alpha_t^2 < \infty$.
% MARTHAC: No, I think inequality is good
  %\textcolor{red}{(DO WE WANT A STRICT INEQUALITY INSTEAD?)}
  Assume that $(\xvec_t, R_t, \xvec_{t+1}, \rho_t)$ is an i.i.d. sequence with uniformly bounded second moments for states and rewards, $\Avec + \beta \Ivec$ and $\Cvec$ are non-singular, and that the standard coverage assumption (Sutton \& Barto, 2018) holds, i.e. $b(A|S) > 0 \;\; \forall S, A$ where $\pi(A|S) > 0$.
  % MARTHAC: We still need this, since we allow for beta = 0
  %\textcolor{red}{(IS CONDITION ON $\Cvec + \beta \Ivec$ NEEDED?)} as defined in Section \ref{sec: expected_updates},
   Then $\wvec_t$ converges with probability one to the TD fixed point if \textbf{either} of the following are satisfied:

  (i) $\Avec$ is positive definite, \textbf{or}

  (ii) $\beta <  -\lambda_{\text{max}} (\Hvec^{-1} \Avec \Avec^\top)$ and $\eta > -\lambda_{\text{min}} (\Cvec^{-1} \Hvec)$, with $\Hvec \defeq \frac{\Avec + \Avec^\top}{2}$.
  Note that when $\Avec$ is not positive definite, $-\lambda_{\text{max}} (\Hvec^{-1} \Avec \Avec^\top)$ and $-\lambda_{\text{min}} (\Cvec^{-1} \Hvec)$ are guaranteed to be positive real numbers.
%
 % TDRC achieves convergence even if $\Cvec$ is singular, although with a different set of conditions on $\eta$ and $\beta$ (given in Appendix \ref{sec: proof_main_thm}).
\end{thm}
We can extend this result to allow for singular $\Cvec$, which was not possible for TDC. The set of conditions on $\eta$ and $\beta$, however, are more complex. We include this result in Appendix \ref{sec: convergence_proof_C_singular}, with conditions given in Eq. \ref{eq: alternate_cond_tdrc_convergence_C_singular}.
%\begin{corollary} \label{cor: tdrc_convergence_C_singular}
%  If $\Cvec$ is not invertible, then TDRC still converges, under a set of conditions on $\eta$ and $\beta$ as given in Eq. \ref{eq: alternate_cond_tdrc_convergence_C_singular} (Appendix \ref{sec: convergence_proof_C_singular}).
%\end{corollary}

Theorem \ref{thm:convergence_tdrc} shows that TDRC maintains convergence when TD is convergent: the case when $\Avec$ is positive definite. Otherwise, TDRC converges under more general settings than TDC, because it has the same conditions on $\eta$ as given by Maei (2011) but allows for $\beta > 0$.
The upper bound on $\beta$ makes sense, since as $\beta \rightarrow \infty$, TDRC approaches TD.
Examining the proof, it is likely that the conditions on $\eta$ could actually be relaxed (see Eq. \ref{eq: third_cond}).
% MARTHAC: Its mostly expalined in the below paragraph.
%\textcolor{red}{\textbf{WHERE TO INCLUDE THIS:} One nice property of the TDRC matrix, as compared with TDC, is that we do not need to require invertibility of $\Cmat$. Rather, by selecting even a very small $\beta > 0$, we can guarantee $\Cmat_\beta$ is invertible.}

One advantage of TDRC is that the matrix $\Cvec_\beta = \Cvec + \beta \Ivec$ is non-singular by construction.
This raises the question: could we have simply changed the MSPBE objective to use $\Cvec_\beta$ and derived the corresponding TDC-like algorithm?
This is easier than TDRC, as the proof of convergence for the resulting algorithm trivially extends the proof from Maei (2011), as the change to the objective function is minimal.
We derive corresponding TDC-like update and demonstrate that it performs notably worse than TDRC in Appendix~\ref{batch_setting}.

\begin{figure*}[t]
  \centering
  \includegraphics[width=\textwidth]{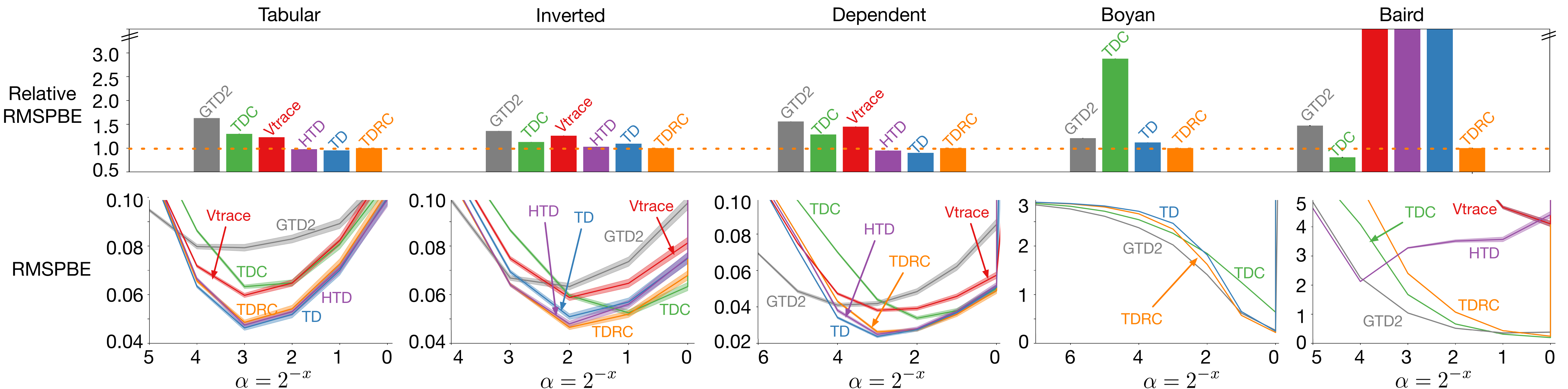}
  \caption{
    \textbf{Top:} The normalized average area under the RMSPBE learning curve for each method on each problem.
    Each bar is normalized by TDRC's performance so that each problem can be shown in the same range.
    All results are averaged over 200 independent runs with standard error bars shown at the top of each rectangle, though most are vanishingly small.
    TD and VTrace both diverge on Baird's Counterexample, which is represented by the bars going off the top of the plot.
    HTD's bar is also off the plot due to its oscillating behavior.
    \textbf{Bottom:} Stepsize sensitivity measured using average area under the RMSPBE learning curve for each method on each problem.
    HTD and VTrace are not shown in Boyan's Chain because they reduce to TD for on-policy problems. Values for bar graphs are given in Table \ref{tab:adagrad_stepsize}.
  }
  \label{fig:alpha_sensitivity_plus_bar_plot}
\end{figure*}

\section{Experiments in the Prediction Setting}\label{sct:ExperimentalResults}

We first establish the performance of TDRC across several small linear prediction tasks where we carefully sweep hyper-parameters, analyze sensitivity, and average over many runs.
The goal is to understand if TDRC has similar performance to TD, with similar parameter sensitivity, but avoids divergence. Before running TDRC, we set $\beta = 1.0$ across all the experiments to refrain from tuning this additional parameter.
\blfootnote{\noindent Code for all experiments is available at: \\ \href{https://github.com/rlai-lab/Regularized-GradientTD}{https://github.com/rlai-lab/Regularized-GradientTD}}

\subsection{Prediction Problems}
In the prediction setting, we investigate three different problems with variations in feature representations, target and behavior policies. We choose problems that have been used in prior work empirically investigating TD methods. The first problem, Boyan's chain (Boyan, 2002), is a 13 state Markov chain where each state is represented by a compact feature representation. This encoding causes inappropriate generalization during learning, but $v_\pi$ can be represented perfectly with the given features.

The second problem is Baird's (1995) well-known star counterexample.
In this MDP, the target and behavior policy are very different resulting in large importance sampling corrections.
Baird's Counterexample has been used extensively to demonstrate the soundness of Gradient TD algorithms, so provides a useful testbed to demonstrate that TDRC does not sacrifice soundness for ease-of-use.

Finally, we include a five state random walk MDP. We use three different feature representations: tabular (unit basis vectors), inverted, and dependent features. This last problem was chosen so that we could exactly mirror the experiments used in prior work benchmarking TDC, GTD2, and TD (Sutton et al., 2009). Like Hackman (2012), we used an off-policy variant of the problem. The behavior policy chooses the left and right action with equal probability, and the target policy chooses the right action 60\% of the time.
Figure \ref{fig:envs} in the appendix summarizes all three problems.

We report the total RMSPBE over 3000 steps, measured on each time step, averaged over 200 independent runs. The learning algorithms under study have tunable meta-parameters that can dramatically impact the efficiency of learning. We extensively sweep the values of these meta-parameters (as described in Appendix \ref{Parameter-settings}), and report both summary performance and the sensitivity of each method to its meta-parameters.
For all results reported in the prediction setting, we use the Adagrad (Duchi, Hazan \& Singer, 2011) algorithm to adapt a vector of stepsizes for each algorithm.
Additional results for constant scalar stepsizes and ADAM vector stepsizes can be found in Appendix~\ref{app_acc} and Appendix~\ref{sct:AdditionalResults}; the conclusions are similar.

\begin{figure*}[ht]
  \centering
  \includegraphics[width=\textwidth]{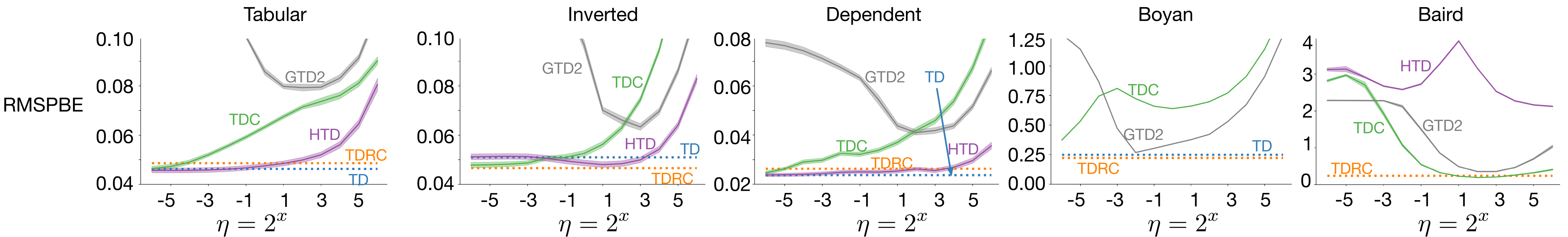}
  \caption{
    Sensitivity to the second stepsize, for changing parameter $\eta$. All methods use Adagrad. All methods are free to choose any value of $\alpha$ for each $\eta$. Methods that do not have a second stepsize are shown as a flat line.
    Values swept are $\eta \in \{2^{-6}, 2^{-5}, \ldots , 2^5, 2^6 \}$.
  }
  \label{fig:eta_sensitivity_adagrad}
\end{figure*}

\begin{figure*}[t]
  \centering
  \includegraphics[width=\textwidth]{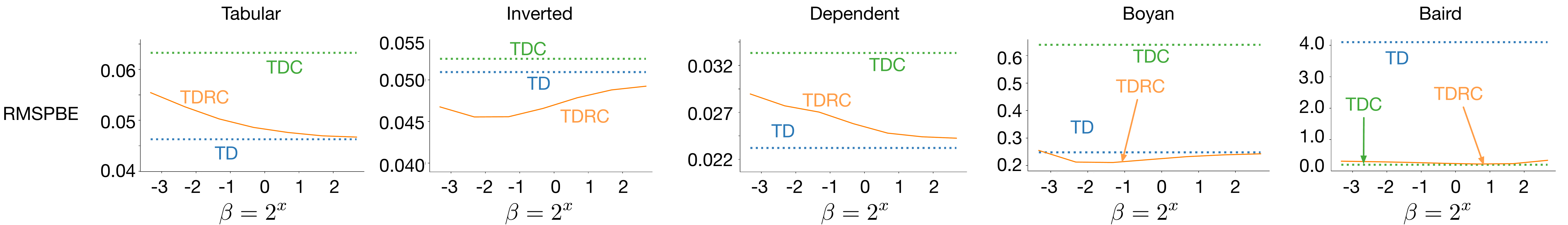}
  \caption{
    Sensitivity to the regularization parameter, $\beta$.
    TD and TDC are shown as dotted baselines, demonstrating extreme values of $\beta$; $\beta = 0$ represented by TDC and $\beta \rightarrow \infty$ represented by TD.
    This experiment demonstrates TDRC's notable insensitivity to $\beta$.
    Its similar range of values across problems, including Baird's counterexample, motivates that $\beta$ can be chosen easily and is not heavily problem dependent.
    Values swept are: $\beta \in 0.1 * \{2^0, 2^1, \ldots , 2^5, 2^6 \}$.
  }
  \label{fig:beta_sensitivity_main}
\end{figure*}

\subsection{Overall Performance}\label{sec_overall}

We first report performance for both the best stepsize as well as provide the parameter sensitivity plots in Figure \ref{fig:alpha_sensitivity_plus_bar_plot}.
In the bar plot, we compactly summarize relative performance to TDRC. TDRC performs well across problems, while every other method has at least one setting where it does noticeably worse than TDRC. GTD2 generally learns more slowly than other methods. This result is unsurprising, as it relies so heavily on $\vech$ for learning $\vecw$: $\vecw_{t+1} \gets \vecw_t + \alpha (\vecx_t - \gamma \vecx_{t+1}) \vech_t^{\tr}\vecx_t$. In the beginning, when $\vech$ is inaccurate, the updates for $\vecw$ are poor. TDC generally learns much faster. In Boyan's chain, however, TDC seems to suffer from variance in $\vech$. The features in this environment cause bigger changes in $\vech$ than in the other environments. TDRC, on the other hand, which regularizes $\vech$, significantly improves learning in Boyan's chain. TD and HTD perform very well across all problems except for Baird's. Finally, Vtrace---which uses a TD update with importance sampling ratios clipped at 1---performs slightly worse than TD due to the introduced bias, but does not mitigate divergence issues due to off-policy learning in Baird's.

The results reported here for TDC do not match previous results which indicate performance generally as good as TD (White \& White, 2016). The reason for this discrepancy is that previous results carefully tuned the second stepsize $\eta\alpha$ for TDC. The need to tune $\eta$ is part of the difficulty in using TDC. To better understand the role it is playing here, we include an additional result where we sweep $\eta$ as well as $\alpha$ for TDC; for completeness, we also include this sweep for GTD2 and HTD. We sweep $\eta \in \{2^{-6}, 2^{-5}, \ldots , 2^5, 2^6 \}$. This allows for $\eta\alpha$ that is very near zero as well as $\eta\alpha$ much larger than $\alpha$. The theory for TDC suggests $\eta$ should be larger than 1. The results in Figure~\ref{fig:eta_sensitivity_adagrad}, however, demonstrate that TDC almost always prefers the smallest $\eta$; but for very small $\eta$ TDC is effectively a TD update. By picking a small $\eta$, TDC essentially keeps $\vech$ near zero---its initialization---and so removes the gradient correction term. TDC was therefore able to match TD by simply tuning a parameter so that it effectively \textit{was} TD. Unfortunately, this is not a general strategy, for instance in Baird's, TDC picks $\eta \ge 1$ and small $\eta$ perform  poorly.

\subsection{Sensitivity to $\beta$}
So far we have only used TDRC with a regularization parameter $\beta = 1$.
This choice was both to avoid over-tuning our method, as well as to show that an intuitive default value could be effective across settings. Intuitively, TDRC should not be sensitive to $\beta$, as both TDC ($\beta = 0$) and TD (large $\beta$) generally perform reasonably. Picking a $\beta > 0$ should enable TDRC to learn faster like TD---by providing a lower variance correction---as long as it's not too large, to ensure we avoid the divergence issues of TD.

We investigate this intuition by looking at performance across a range of $\beta \in 0.1 * \{2^0, 2^1, \ldots , 2^5, 2^6 \}$. For $\beta = 0$, we have TDC. Ideally, performance should quickly improve for any non-negligible $\beta$, with a large flat region of good performance in the parameter sensitivity plots for a wide range of $\beta$. This is generally what we observe in Figure~\ref{fig:beta_sensitivity_main}. For even very small $\beta$, TDRC noticeably improves performance over TDC, getting halfway between TDC and TD (Random Walk with Tabular or Dependent features) or in some cases immediately obtaining the good performance of TD (Random Walk with Inverted Features, Boyan's chain and Baird's). Further, in these three cases, it even performs better or comparably to both TDC and TD for all tested $\beta$. Notably, these are the settings with more complex feature representations, suggesting that the regularization parameter helps TDRC learn an $\vech$ that is less affected by harmful aliasing in the feature representation. Finally, the results also show that $\beta = 1.0$ was in fact not optimal, and we could have obtained even better results in the previous section, typically with a larger $\beta$. These improvements, though, were relatively marginal over the choice of $\beta = 1.0$.

Naturally, the scale of $\beta$ should be dependent on the magnitude of the rewards, because in TDRC the gradient correction term is attempting to estimate the expected TD error. One answer is to simply employ adaptive target normalization, such as Pop-Art (van Hasselt et al., 2016), and keep $\beta$ equal to one. We found TDRC with $\beta=1$ performed at least as well as TD in on-policy chain domains across a large range of reward scales (see Appendix \ref{sec:reward_scale_sensitivity}).

\section{Experiments in the Control Setting}

Like TD, TDRC was developed for prediction, under linear function approximation. Again like TD, there are natural---though in some cases heuristic---extensions to the control setting and to non-linear function approximation. In this section, we investigate if TDRC can provide similar improvements in the control setting. We first investigate TDRC in control with linear function approximation, where the extension is more straightforward. We then provide a heuristic strategy to use TDRC---and TDC---with non-linear function approximation. We demonstrate, for the first time, that Gradient TD methods can outperform Q-learning when using neural networks, in two classic control domains and two visual games.

\subsection{Extending TDRC to Control}

Before presenting the control experiments, we describe how to extend TDRC to control, and to non-linear function approximation. The extension to non-linear function approximation is also applicable in the prediction setting; we therefore begin there. We then discuss the extension to Q-learning which involves estimating action-values for the greedy policy.

Consider the setting where we estimate $\hat v(s)$ using a neural network. The secondary weights in TDRC are used to obtain an estimate of $\E{\delta_t | S_t = s}$. Under linear function approximation, this expected TD error is estimated using linear regression with $\ell_2$ regularization: $\vech^\top\vecx_t \approx \E{\delta_t | S_t = s}$. With neural networks, this expected TD error can be estimated using an additional head on the network. The target for this second head is still $\delta_t$, with a squared error and $\ell_2$ regularization. One might even expect this estimate of $\E{\delta_t | S_t = s}$ to improve, when using a neural network, rather than a hand-designed basis.

An important nuance is that gradients are not passed backward from the error in this second head. This choice is made for simplicity, and to avoid any issues when balancing these two losses. The correction is secondary, and we want to avoid degrading performance in the value estimates simply to improve estimates of $\E{\delta_t | S_t = s}$. It also makes the connection to TD more clear as $\beta$ becomes larger, as the update to the network is only impacted by $\vecw$. We have not extensively tested this choice; it remains to be seen if using gradients from both heads might actually be a better choice.

The next step is to extend the algorithm to action-values. For an input state $s$, the network produces an estimate $\hat q(s, a)$ and a prediction $\hat\delta(s,a)$ of $\E{\delta_t | S_t = s, A_t = a}$ for each action.
The weights $\vech_{t+1, A_t}$ for the head corresponding to action $A_t$ are updated using the features produced by the last layer $\xvec_t$, with $\hat\delta(S_t,A_t) = \vech_{t,A_t}^{\tr} \vecx_t$:
\begin{align}
    \vech_{t+1,A_t} \leftarrow& ~\vech_{t,A_t} + \alpha \bigl[\delta_t - \vech_{t,A_t}^{\tr} \vecx_t \bigr] \vecx_t - \alpha \beta \vech_{t,A_t} \label{eq:secondary_weight_dqrc}
\end{align}
For the other actions, the secondary weights are not updated since we did not get a target $\delta_t$ for them.

The remaining weights $\vecw_{t}$, which include all the weights in the network excluding $\vech$, are updated using
\begin{align}
&\delta_t = R_{t+1} + \gamma {q(S_{t+1}, a')} - q(S_t, A_t)\label{eq:primary_weight_dqrc}\\
  & \vecw_{t+1} \!\!\leftarrow \!\!\vecw_t\! + \!\alpha\delta_t \gradSA{S_t}{A_t}
    \!-\!  \alpha\gamma \hat\delta(S_t,A_t)\gradSA{S_{t+1}}{a'}  \nonumber
\end{align}
where $a'$ is the action that the policy we are evaluating would take in state $S_{t+1}$. For control, we often select the greedy policy, and so $a' = \argmax_a {q(S_{t+1}, a)}$ and
$\delta_t = R_{t+1} + \gamma \max_a {q(S_{t+1}, a)} - q(S_t, A_t)$ as in Q-learning. This action $a'$ may differ from the (exploratory) action $A_{t+1}$ that is actually executed, and so this estimation is off-policy.  There are no importance sampling ratios because we are estimating action-values.

We call this final algorithm QRC: Q-learning with Regularized Corrections. The secondary weights in QRC are initialized to $\vec0$, to maintain the similarity to TD. We can obtain, as a special case, a control algorithm based on TDC, which we call QC. If we set $\beta = 0$ in Eq. \ref{eq:secondary_weight_dqrc}, we obtain QC.

We conclude this section by highlighting that there is an alternative route to use TDRC, as is, for control: by using TDRC as a critic within Actor-Critic. We provide the update equations in Appendix \ref{app_ac_tdrc}.

\subsection{Control Problems}
We first test the algorithms in a well-understood setting, in which we know Q-learning is effective: Mountain Car with a tile-coding representation. We then use neural network function approximation in two classic control environments---Mountain Car and Cart Pole---and two visual environments from the MinAtar suite (Young \& Tian, 2019). For all environments, we fix $\beta = 1.0$ for QRC, $\eta = 1.0$ for QC and do not use target networks (for experiments with target networks see Appendix~\ref{app:TargetNets}).

In the two classic control environments, we use 200 runs, an $\epsilon$-greedy policy with $\epsilon = 0.1$ and a discount of $\gamma = 0.99$.
In Mountain Car (Moore, 1990; Sutton, 1996), the goal is to reach the top of a hill, with an underpowered car. The state consists of the agent's position and velocity, with a reward of $-1$ per step until termination, with actions to accelerate forward, backward or do nothing. In Cart Pole (Barto, Sutton \& Anderson, 1983), the goal is to keep a pole balanced as long as possible, by moving a cart left or right. The state consists of the position and velocity of the cart, and the angle and angular velocity of the pole. The reward is +1 per step. An episode ends when the agent fails to balance the pole or balances the pole for more than 500 consecutive steps. For non-linear control experimental details on these environments see Appendix~\ref{app:CPAndMCExpDetails}.

For the two MinAtar environments, Breakout and Space Invaders, we use 30 runs, $\gamma = 0.99$ and a decayed $\epsilon$-greedy policy with $\epsilon=1$ decaying linearly to $\epsilon=0.1$ over the first 100,000 steps. In Breakout, the agent moves a paddle left and right, to hit a ball into bricks. A reward of +1 is given for every brick hit; new rows appear when all the rows are cleared. The episode ends when the agent misses the ball and it drops.
In Space Invaders, the agent shoots alien ships coming towards it, and dodges their fire. A reward of +1 is given for every alien that is shot. The episode ends when the spaceship is hit by alien fire or reached by an alien ship. These environments are simplified versions from the Atari suite, designed to avoid the need for large networks and make it more feasible to complete more exhaustive comparison, including using more runs. All methods use a network with one convolutional layer, followed by a fully connected layer. All experimental settings are identical to the original MinAtar paper (see Appendix \ref{app:MinAtarExpDetails} for details).

\subsection{Linear Control}

We compare TD, TDC and TDRC for control, both within an Actor-Critic algorithm and with their extensions to Q-learning. In Figure~\ref{fig:linear-mc}, we can see two clear outcomes from both control experiments. In both cases, the control algorithm based on TDC fails to converge to a reasonable policy. The TDRC variants, on the other hand, match the performance of TD.

\begin{figure}[h]
  \centering
  \includegraphics[width=0.45\textwidth]{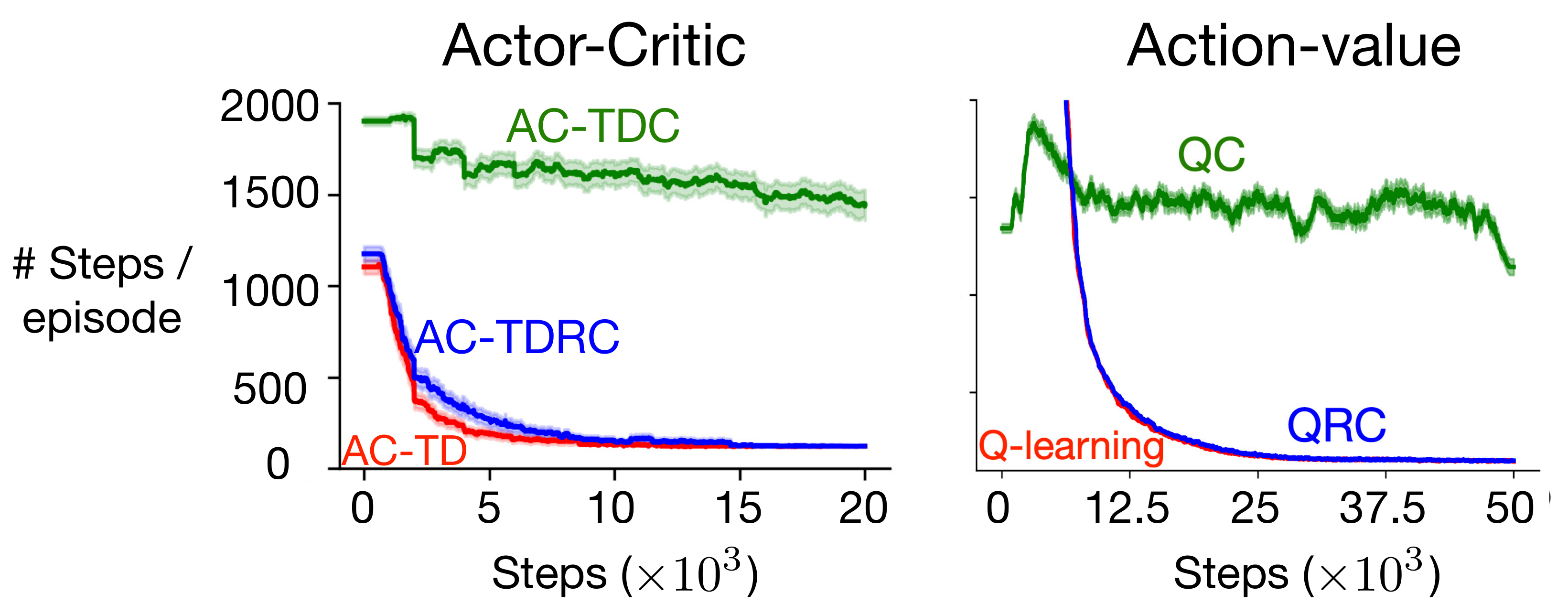}
  \caption{
    Numbers of steps to reach goal, averaged over runs, versus number of environment steps, in Mountain Car with tile-coded features.
    \textbf{Left:} Comparison of actor-critic control algorithms with various critics with ADAM optimizer. For actor critic experimental details see Appendix~\ref{app_ac_tdrc}.
    \textbf{Right:} Comparison of state-action value control algorithms with constant stepsizes.
    Stepsizes were swept over $\alpha \in \{2^{-8}, 2^{-7}, \ldots, 2^{-2}, 2^{-1} \}$ and then scaled by the number of active features. We used 16 tilings and $4 \times 4$ tiles. Results are averaged over 200 independent runs, with shaded error corresponding to standard error.
  }
  \label{fig:linear-mc}
\end{figure}

This result might be surprising, since the only difference between TDRC and TDC is regularizing $\vech$. This small addition, though, seems to play a big role in avoiding this surprisingly bad performance of TDC, and potentially explains why gradient methods have been dismissed as hard-to-use.
When we looked more closely at TDC's behavior, we found that the TDC agent improved its behavior policy quickly. But, the magnitude of the gradient corrections also grew rapidly. This high magnitude gradient correction resulted in a higher magnitude gradient for $\vecw$, and pushed down the learning rate for TDC.
The constraint on this correction term provided by TDRC seems to prevent this explosive growth, allowing TDRC to attain comparable performance to the TD-based control agent.

\subsection{Non-linear Control}
\label{subsct:LargeScaleControlExperiments}

When moving to non-linear function approximation, with neural networks, we find a more nuanced outcome: QC still suffers compared to Q-learning and QRC in the classic control environments---though less than before---yet provides substantial improvements in the two MinAtar environments.

In Figure~\ref{fig:NonlinearMCAndCP}, we find that QC learns more slowly than QRC and Q-learning. Again, QRC brings performance much closer to Q-learning, when QC is performing notably more poorly. In Mountain Car, we tested a more highly off-policy setting: 10 replay steps. By using more replay per step, more data from older policies is used, resulting in a more off-policy data distribution. Under such an off-policy setting, we expect Q-learning to suffer, and in fact, we find that QRC actually performs better than Q-learning. We provide additional experiments on Mountain Car in Appendix~\ref{sec: qc_mc_non_linear_investigation}.

On the two MinAtar environments, in Figure~\ref{fig:breakout}, we obtain a surprising result: QC provides substantial performance improvements over Q-learning. QRC with $\beta=1$ is not as performant as QC in this setting and instead obtains performance in-between QC and Q-learning. However, QRC with smaller values of regularization parameter (shown as lighter blue lines) results in the best performance. This outcome highlights that Gradient TD methods are not only theoretically appealing, but could actually be a better alternative to Q-learning in standard (non-adversarially chosen) problems. It further shows that, though QRC with $\beta = 1.0$ generally provides a reasonable strategy, substantial improvements could be obtained with an adaptive method for selecting $\beta$.

\begin{figure}[t]
  \centering
  \includegraphics[width=0.43\textwidth]{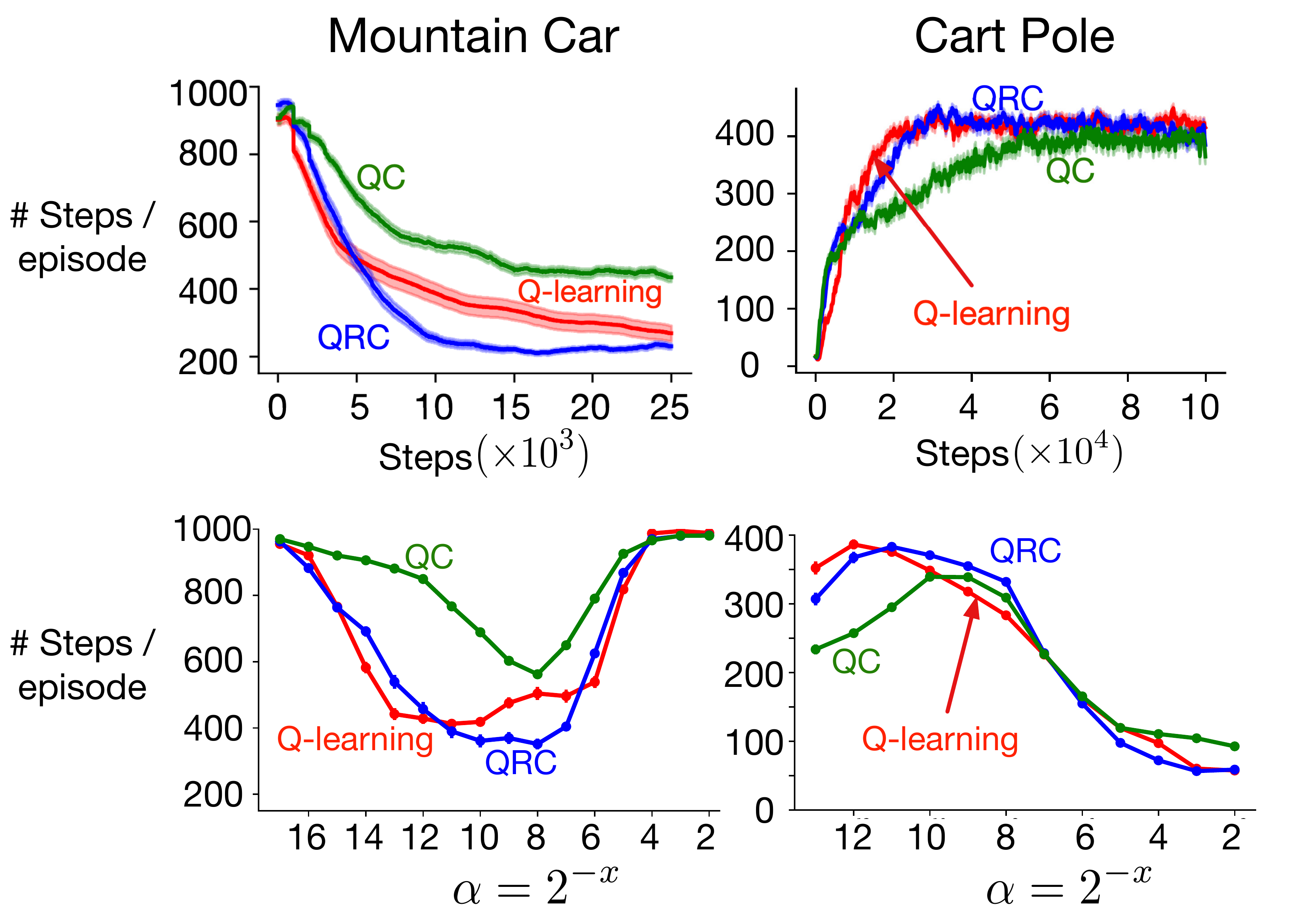}
  \caption{
    Performance of Q-learning, QC and QRC on two classic control environments. On top the learning curves are shown and at the bottom the parameter sensitivity for various stepsizes. Lower is better for Mountain Car (fewer steps to goal) and higher is better for Cart Pole (more steps balancing the pole). Results are averaged over 200 runs, with shaded error corresponding to standard error.
  }
  \label{fig:NonlinearMCAndCP}
\end{figure}

\begin{figure}[t]
  \centering
  \includegraphics[width=0.43\textwidth]{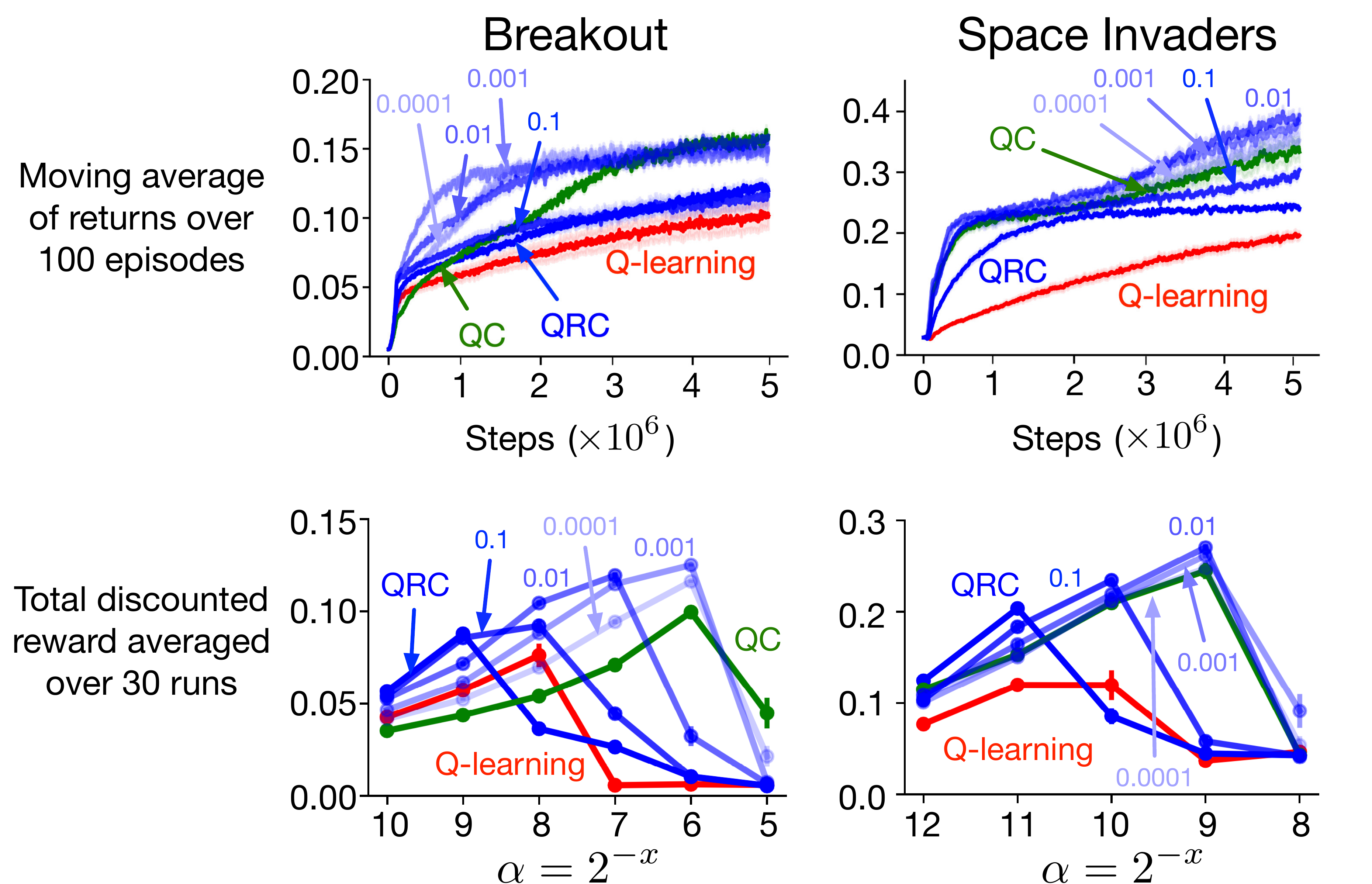}
  \caption{
    Performance of Q-learning, QC, and QRC in the two MinAtar environments. The learning curves in the top row depict the average return over time for the best performing stepsize for each agent. The stepsize sensitivity plots in the bottom row depict the total discounted reward achieved with several stepsize values. Higher is better. Results are averaged over 30 independent runs, with shaded error corresponding to standard error. Light blue lines show the performance of QRC with smaller regularization parameters, $\beta < 1$.
  }
  \label{fig:breakout}
\end{figure}

\section{Conclusions and Discussion}
\label{sct:ConclusionsAndFutureWork}

In this work, we introduced a simple modification of the TDC algorithm that achieves performance much closer to that of TD. Our algorithm uses a single stepsize like TD, and behaves like TD when TD performs well but also prevents divergence under off-policy sampling. TDRC is built on TDC, and, as we prove, inherits its soundness guarantees.
In small linear prediction problems TDRC performs best overall and exhibits low sensitivity to its regularization parameter. In control experiments, with extensions to non-linear function approximation, we find that the resulting algorithm, QRC, performs as well as Q-learning and in some cases notably better. This constitutes the first demonstration of Gradient TD methods outperforming Q-learning, and suggests this simple modification to the standard Q-learning update---to give QRC---could provide a more general purpose algorithm.

An important next step is to better understand the conditions on the regularization parameter $\beta$ and whether we can truly remove the second stepsize $\eta$. The current theorem does not remove conditions on $\eta$; in fact, it has the same conditions as TDC. We hypothesize that $\beta$ should make $\vech$ converge more quickly, and so remove the need for the stepsize for the secondary weights to be bigger. Further, the conditions on $\eta$ and $\beta$ both depend on domain specific quantities that are generally difficult to compute. In the small prediction problems, we were easily able to confirm that our choices of meta-parameter met the theoretical conditions, however for the larger control problems this remains an open question. In general, developing tight conditions on $\eta$ and $\beta$ would help facilitate comfort in using TDRC.

Another important next step is to thoroughly investigate if these empirical results hold in a broader range of environments and settings. The results in this work suggest that TDRC could potentially be a replacement for the widely used TD algorithms. It is only a small modification to an existing TD implementation, and so would not be difficult to adopt. But, to make such a bold claim, much more evidence is needed, particularly because TD has been shown to be so successful for many years.

\newpage
\section*{Acknowledgments}
This work was funded by NSERC and CIFAR, particularly through funding the Alberta Machine Intelligence Institute (Amii) and the CCAI Chair program. The authors also gratefully acknowledge funding from JPMorgan Chase \& Co. and Google DeepMind. We would like to thank Csaba Szepesv\'ari, Anna Harutyunyan, and the anonymous reviewers for useful feedback. We also thank Banafsheh Rafiee, Andrew Jacobsen, and Alan Chan for helpful discussions during the course of this project.

\section*{References}

\begin{list}{}{%
    \setlength{\topsep}{0pt}%
    \setlength{\leftmargin}{0.2in}%
    \setlength{\listparindent}{-0.2in}%
    \setlength{\itemindent}{-0.2in}%
    \setlength{\parsep}{\parskip}%
  }%
\item[]
  Baird, L. C. (1995). Residual algorithms: Reinforcement learning with function approximation. In \textit{International Conference on Machine Learning}, pp. 30–37. Morgan Kaufmann, San Francisco.

\item[]
  Barto, A. G., Sutton, R. S., Anderson, , C. W. (1983). Neuronlike adaptive elements that can solve difficult learning control problems. \textit{IEEE transactions on systems, man, and cybernetics, 5}, 834-846.

\item[]
  Bellemare, M. G., Naddaf, Y., Veness, J., Bowling, M. (2013). The arcade learning environment: An evaluation platform for general agents. \textit{Journal of Artificial Intelligence Research, 47}, 253-279.

\item[]
  Borkar, V. S., Meyn, S.P. (2000). The {O.D.E.} Method for Convergence of Stochastic Approximation and Reinforcement Learning. \textit{{SIAM} J. Control and Optimization}.

\item[]
  Boyan, J.A. (2002). Technical Update: Least-Squares Temporal Difference Learning. \textit{Machine Learning}.

\item[]
  Dai, B., Albert, S., Lihong, L., Lin, X., Niao, H., Zhen, L., Jianshu, C., Le, S. SBEED: Convergent reinforcement learning with nonlinear function approximation. In \textit{International Conference on Machine Learning} (pp. 1125-1134).

\item[]
  Du, S. S., Chen, J., Li, L., Xiao, L., and Zhou, D. (2017) Stochastic variance reduction methods for policy evaluation. In \textit{International Conference on Machine Learning.}

\item[]
  Duchi, J., Hazan, E., Singer, Y. (2011). Adaptive subgradient methods for online learning and stochastic optimization. \textit{Journal of Machine Learning Research 12}:2121-2159.

\item[]
  Espeholt, L., Soyer, H., Munos, R., Simonyan, K., Mnih, V., Ward, T., Doron, Y., Firoiu, V., Harley, T., Dunning, I. and Legg, S. (2018) IMPALA: Scalable distributed Deep-RL with importance weighted actor-learner architectures. In \textit{International Conference on Machine Learning.}

\item[]
Feng, Y., Li, L., Liu, Q. (2019). A kernel loss for solving the bellman equation. In \textit{Advances in Neural Information Processing Systems} (pp. 15430-15441).

\item[]
  Ghiassian, S., Patterson, A., White, M., Sutton, R. S., White, A. (2018). Online off-policy prediction. ArXiv:1811.02597.

\item[]
  Glorot, X., Bengio, Y. (2010). Understanding the difficulty of training deep feedforward neural networks. In \textit{International Conference on Artificial Intelligence and Statistics} (pp. 249-256).

\item[]
  Hackman, L. (2012). \textit{Faster Gradient TD Algorithms}. M.Sc. thesis, University of Alberta, Edmonton.

\item[]
  Juditsky, A., Nemirovski, A. (2011). \textit{Optimization for Machine Learning}.

\item[]
Kingma, D. P., Ba, J. (2014). Adam: A method for stochastic optimization. arXiv preprint arXiv:1412.6980.

\item[]
  Liu B, Liu J, Ghavamzadeh M, Mahadevan S, Petrik M (2015). Finite-Sample Analysis of Proximal Gradient TD Algorithms. In \textit{International Conference on Uncertainty in Artificial Intelligence}, pp. 504-513.

\item[]
  Liu B, Liu J, Ghavamzadeh M, Mahadevan S, Petrik M (2016). Proximal Gradient Temporal Difference Learning Algorithms. In \textit{International Joint Conference on Artificial Intelligence}, pp. 4195-4199.

\item[]
  Mahadevan, S., Liu, B., Thomas, P., Dabney, W., Giguere, S., Jacek, N., Gemp, I., Liu, J. (2014). Proximal reinforcement learning: A new theory of sequential decision making in primal-dual spaces. ArXiv:1405.6757.

\item[]
  Mahmood, A. R., Yu, H., Sutton, R. S. (2017). Multi-step off-policy learning without importance sampling ratios. ArXiv:1702.03006.

\item[]
  Maei, H. R. (2011). \textit{Gradient temporal-difference learning algorithms}. Ph.D. thesis, University of Alberta, Edmonton.

\item[]
Mnih, V., Badia, A. P., Mirza, M., Graves, A., Lillicrap, T., Harley, T., Silver, D., Kavukcuoglu, K. (2016). Asynchronous methods for deep reinforcement learning. In \textit{International conference on machine learning} (pp. 1928-1937).

\item[]
  Moore, A. W. (1990). \textit{Efficient memory-based learning for robot control}. Ph.D. theis, University of Cambridge.

\item[]
  Munos, R., Stepleton, T., Harutyunyan, A., Bellemare, M. (2016). Safe and efficient off-policy reinforcement learning. In \textit{Advances in Neural Information Processing Systems 29}, pp. 1046–1054.

\item[]
Precup, D., Sutton, R. S., Dasgupta, S. (2001). Off-policy temporal-difference learning with function approximation. In \textit{Proceedings of the 18th International Conference on Machine Learning}, pp. 417–424.

\item[]
  Reddi, S. J., Kale, S., Kumar, S. (2019). On the convergence of adam and beyond. ArXiv:1904.09237.

\item[]
Schaul, T., Quan, J., Antonoglou, I., Silver, D. (2016). Prioritized experience replay. In \textit{International Conference on  Learning Representations.}

\item[]
Sutton, R. S. (1996). Generalization in reinforcement learning: Successful examples using sparse coarse coding. In \textit{Advances in Neural Information Processing Systems 8 (NIPS 1995),} pp. 1038–1044. MIT Press, Cambridge, MA.

\item[]
  Sutton, R. S., Barto, A. G. (2018). \textit{Reinforcement Learning: An Introduction,} Second Edition. MIT Press.

\item[]
  Sutton, R. S., Maei, H. R., Precup, D., Bhatnagar, S., Silver, D., Szepesv\'ari, Cs., Wiewiora, E. (2009). Fast gradient-descent methods for temporal-difference learning with linear function approximation. In \textit{International Conference on Machine Learning}, pp. 993–1000, ACM.

\item[]
  Sutton, R. S., Mahmood A. R., and White M. (2016) An emphatic approach to the problem of off-policy temporal-difference learning. \textit{The Journal of Machine Learning Research}.

\item[]
  Touati, A., Bacon, P. L., Precup, D., Vincent, P. (2018). Convergent tree-backup and retrace with function approximation. ArXiv:1705.09322.

\item[]
  van Hasselt, H. P., Guez, A., Hessel, M., Mnih, V., Silver, D. (2016). Learning values across many orders of magnitude. In \textit{Advances in Neural Information Processing Systems} (pp. 4287-4295).

\item[]
  van Hasselt, H., Doron, Y., Strub, F., Hessel, M., Sonnerat, N., Modayil, J. (2018). Deep Reinforcement Learning and the Deadly Triad. ArXiv:1812.02648

\item[]
  White, A., White, M. (2016). Investigating Practical Linear Temporal Difference Learning. In \textit{International Conference on Autonomous Agents {\&} Multiagent Systems}.

\item[]
  Young, K., Tian, T. (2019). MinAtar: An Atari-Inspired Testbed for Thorough and Reproducible Reinforcement Learning Experiments. ArXiv:1903.03176.

\end{list}

\newpage
\clearpage

\appendix
\section{Results in the Batch Setting}\label{batch_setting}

The proofs of convergence for many of the methods require independent samples for the updates.
This condition is not generally met in the fully online learning setting that we consider throughout the rest of the paper.
In Figure~\ref{fig:batch_update_sensitivity} we show results for all methods in the fully offline batch setting, demonstrating that---on the small problems that we consider---the conclusions do not change when transferring from the batch setting to the online setting.
We include two additional methods in the batch setting, the Kernel Residual Gradient methods (Feng, Li \& Liu, 2019), which do not have a clear fully online implementation.

We create a new batch dataset for each of 500 independent runs by getting 100k samples from the state distribution induced by the behavior policy, then sampling from the transition kernel for each of these states.
We then perform mini-batch updates by sampling 8 independent transitions from this dataset.
Each algorithm makes $n$ updates for $n \in [1, 2, 4, 8, \ldots, 8192]$, choosing the stepsize which minimizes the area under the RMSPBE learning for each $n$.
This effectively shows the best performance of each algorithm if it was given a budget of $n$ updates, allowing us to make comparisons across several different timescales.
The constant stepsizes swept are $\alpha \in \{2^{-8}, 2^{-7}, \ldots, 2^{0}\}$.

In Figure~\ref{fig:batch_update_sensitivity}, we demonstrate that GTD2 and the Kernel-RG methods generally perform poorly across these set of domains.
We additionally show that TDC, TD, and TDRC are often indistinguishable in the batch setting---except Boyan's Chain where TDC still performs inexplicably poorly---suggesting that perhaps TDRC's gain in performance of TDC is due to the correlated sampling induced by online learning.
We finally show that TDC++, which is TDC with regularized $\Cvec$, generally performs comparably to GTD2.

\subsection{Relationship to Residual Gradients}
The Residual Gradient (RG) family of algorithms provide an alternative gradient-based strategy for performing temporal difference learning.
The RG methods minimize the Mean Squared Bellman Error (MSBE), while the Gradient TD family of algorithms minimize a particular form of the MSBE, the Mean Squared \emph{Projected} Bellman Error (MSPBE).
The RG family of methods generally suffer from difficulty in obtaining independent samples from the environment, leading towards stochastic optimization algorithms which find a biased solution (Sutton \& Barto, 2018).
However, very recent work has generalized the MSBE and proposed an algorithmic strategy to perform unbiased stochastic updates (Feng, Li \& Liu, 2019; Dai et al., 2018).
We compare to the approach in Feng, Li, and Liu (2019) below.

\subsection{Derivation of the TDC++ Update Equations}
In this section, we derive the update equations for TDC++, i.e. TDC with the regularized $\Cvec_{\beta}$ matrix. Consider the MSPBE objective (see Eq. \ref{eq:MSPBEInExpectation}) but with a regularized $\Cvec_{\beta}$:
\begin{align*}
  \text{MSPBE}_{\text{++}}(\vecw_t) &\defeq \E{\delta_t \vecx_t}\tr \left(\E{\vecx_t \vecx_t\tr}^{-1} + \beta \Ivec \right)\E{\delta_t \vecx_t} \\
  &= (\sneg\Amat \vecw + \vecb)\tr \Cmat_{\beta}^\inv (\sneg\Amat \vecw + \vecb).
\end{align*}
The gradient of this objective is $-\frac{1}{2} \nabla_{\wvec} \text{MSPBE}_{\text{++}}(\vecw_t) = \Avec\tr \Cvec_{\beta}^{-1} (\bvec - \Avec \wvec_t) = \E{\delta_t \vecx_t}  - \gamma \E{\vecx' \vecx\tr}\vech_\beta - \beta \hvec_{\beta}$. Using this gradient and the same update for $\hvec_{t+1}$ as in TDRC, we obtain the update equations for TDC++ (with an additional $\eta$ in the stepsize for $\hvec$):
\begin{align*}
  \vech_{t+1} \leftarrow& ~\vech_t + \eta \alpha \bigl[\delta_t - (\vech_t^{\tr} \vecx_t) \bigr] \vecx_t - \eta \alpha \beta \vech_t \\
  \vecw_{t+1} \leftarrow& ~\vecw_t + \alpha  \delta_t \vecx_t - \alpha \gamma (\vech_t^{\tr}\vecx)\vecx_{t+1} - \alpha \beta \vech_t.
\end{align*}

\subsection{Convergence of TDC++}
It is straightforward to show that TDC++ converges to the TD fixed point under very similar conditions as TDC (Maei, 2011). We show the key steps here (for details see Maei (2011) or Appendix \ref{sec: proof_main_thm}). The $\Gvec$ matrix for TDC++ is $\Gvec = \begin{bmatrix} -\eta \Cvec_\beta & - \eta \Avec \\ \Avec^\top - \Cvec_\beta & -\Avec \end{bmatrix}$. If we can show that the real parts of all the eigenvalues of $\Gvec$ are negative, then the algorithm would converge. First note that for an eigenvalue $\lambda \in \mathbb{C}$ of $\Gvec$, $\text{det}(\Gvec - \lambda \Ivec) = \text{det}(\lambda(\Cvec_\beta + \lambda \Ivec) + \Avec (\eta \Avec^\top + \lambda \Ivec))= 0$. Then for some non--zero vector $\zvec \in \mathbb{C}$, $\zvec^* (\lambda(\Cvec_\beta + \lambda \Ivec) + \Avec (\eta \Avec^\top + \lambda \Ivec)) \zvec = 0$. Upon simplifying this, we obtain the following quadratic equation in $\lambda$:
\begin{equation*}
  \|\zvec\|^2 \lambda^2 + (\zvec^* ({\eta} \Cvec_\beta + \Avec) \zvec ) \lambda +  {\eta} \|\Avec \zvec\|^2  = 0.
\end{equation*}
If $\lambda_1$ and $\lambda_2$ are two solutions of this equation, then
\begin{align*}
  \lambda_1 \lambda_2 = \eta \frac{\|\Avec \zvec \|^2}{\|\zvec\|^2}, \quad \quad \lambda_1 + \lambda_2 = - \frac{(\zvec^* (\eta \Cvec_\beta + \Avec) \zvec)}{\|\zvec\|^2}.
\end{align*}
Since, $\lambda_1 \lambda_2 > 0$ and real, the real parts of both $\lambda_1$ and $\lambda_2$ have the same sign. Thus, $\text{Re}(\lambda_1 + \lambda_2) < 0$ would imply that each of $\text{Re}(\lambda_1) < 0$ and $\text{Re}(\lambda_2)< 0$ and we would be done. Assuming $\text{Re}(\lambda_1 + \lambda_2) = - \frac{(\zvec^* (\eta \Cvec_\beta + \Avec) \zvec)^* + (\zvec^* (\eta \Cvec_\beta + \Avec) \zvec)}{2 \|\zvec\|^2} = - \frac{\zvec^* (\eta \Cvec_\beta + \Hvec) \zvec}{\|\zvec\|^2} < 0$, where $\Hvec \defeq \frac{1}{2}(\Avec + \Avec^\top)$, leads to the condition
\begin{equation*}
  \eta > -\lambda_{\text{min}} (\Cvec_\beta^{-1} \Hvec),
\end{equation*}
for TDC++ to converge.

TDC++ differs from TDRC in that it has an extra term $(-\alpha \beta \hvec_t)$ in the update for the weight $\wvec_{t+1}$. Further, unlike TDRC, the convergence of TDC++ doesn't require any conditions on $\beta$.

\begin{figure*}
  \centering

  \begin{subfigure}[t]{.2\textwidth}
      % AUC
      \caption*{Tabular}
      \includegraphics[width=\textwidth]{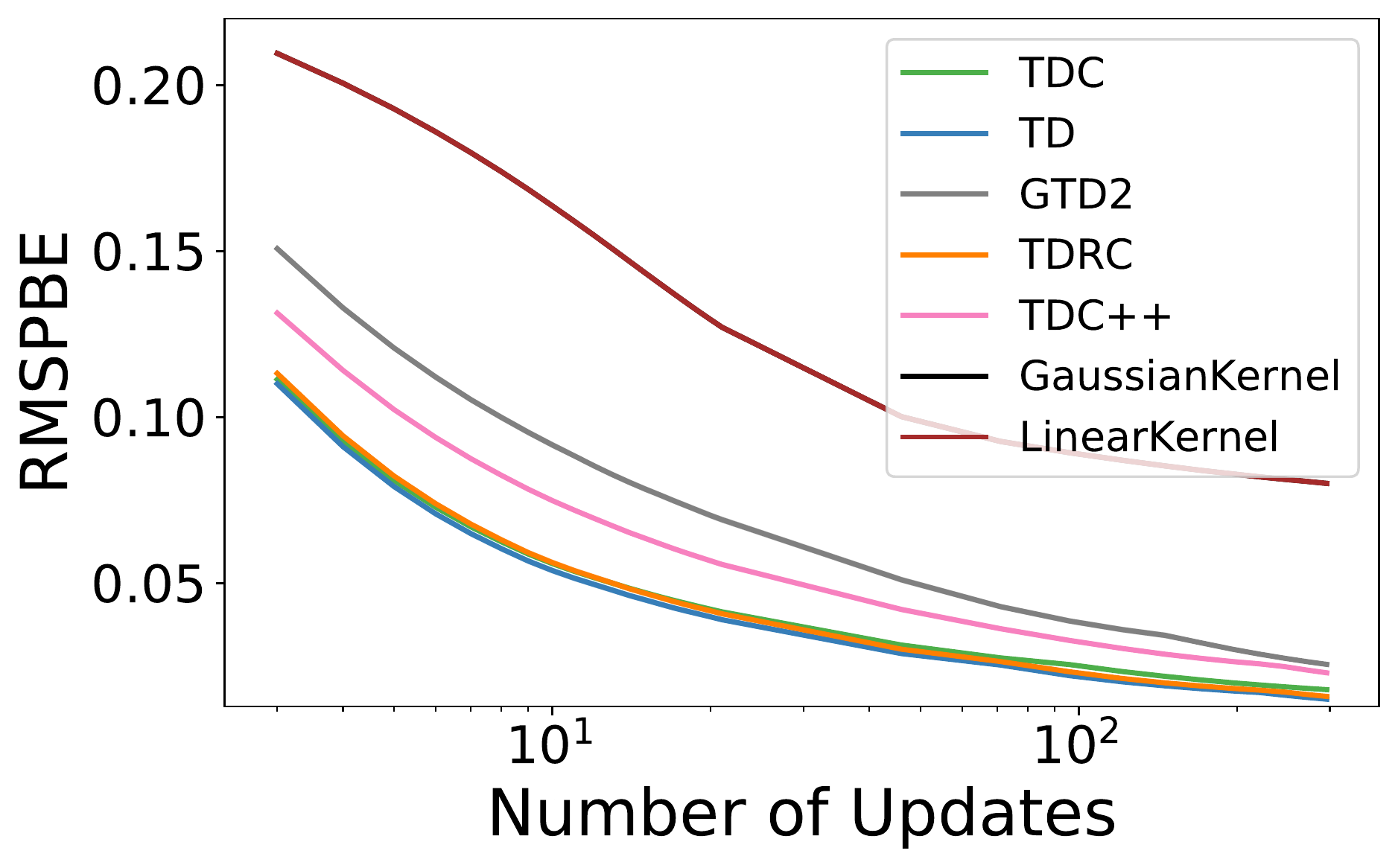}
  \end{subfigure}\hfill
  \begin{subfigure}[t]{.2\textwidth}
      % AUC
      \caption*{Inverted}
      \includegraphics[width=\textwidth]{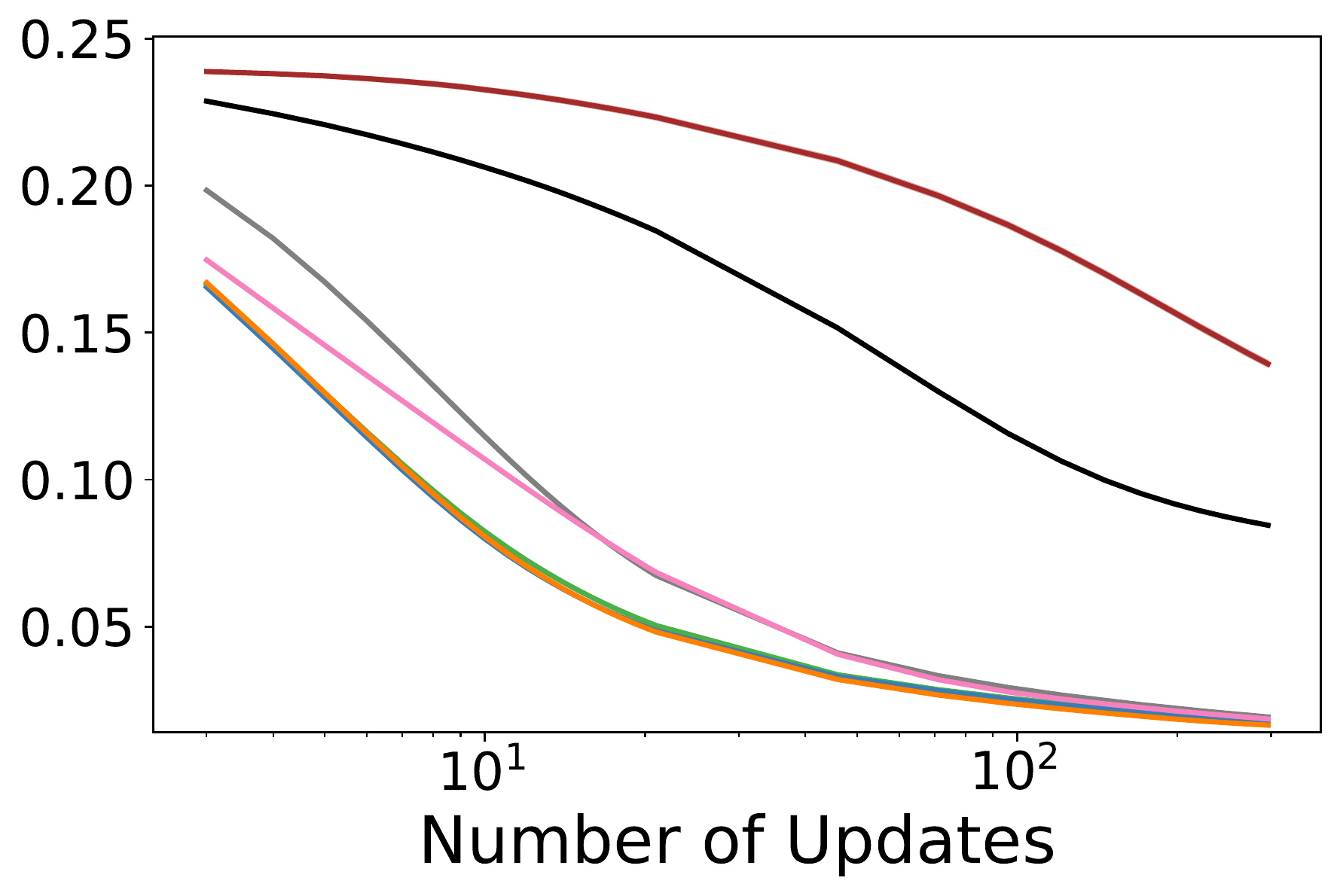}
  \end{subfigure}\hfill
  \begin{subfigure}[t]{.2\textwidth}
      % AUC
      \caption*{Dependent}
      \includegraphics[width=\textwidth]{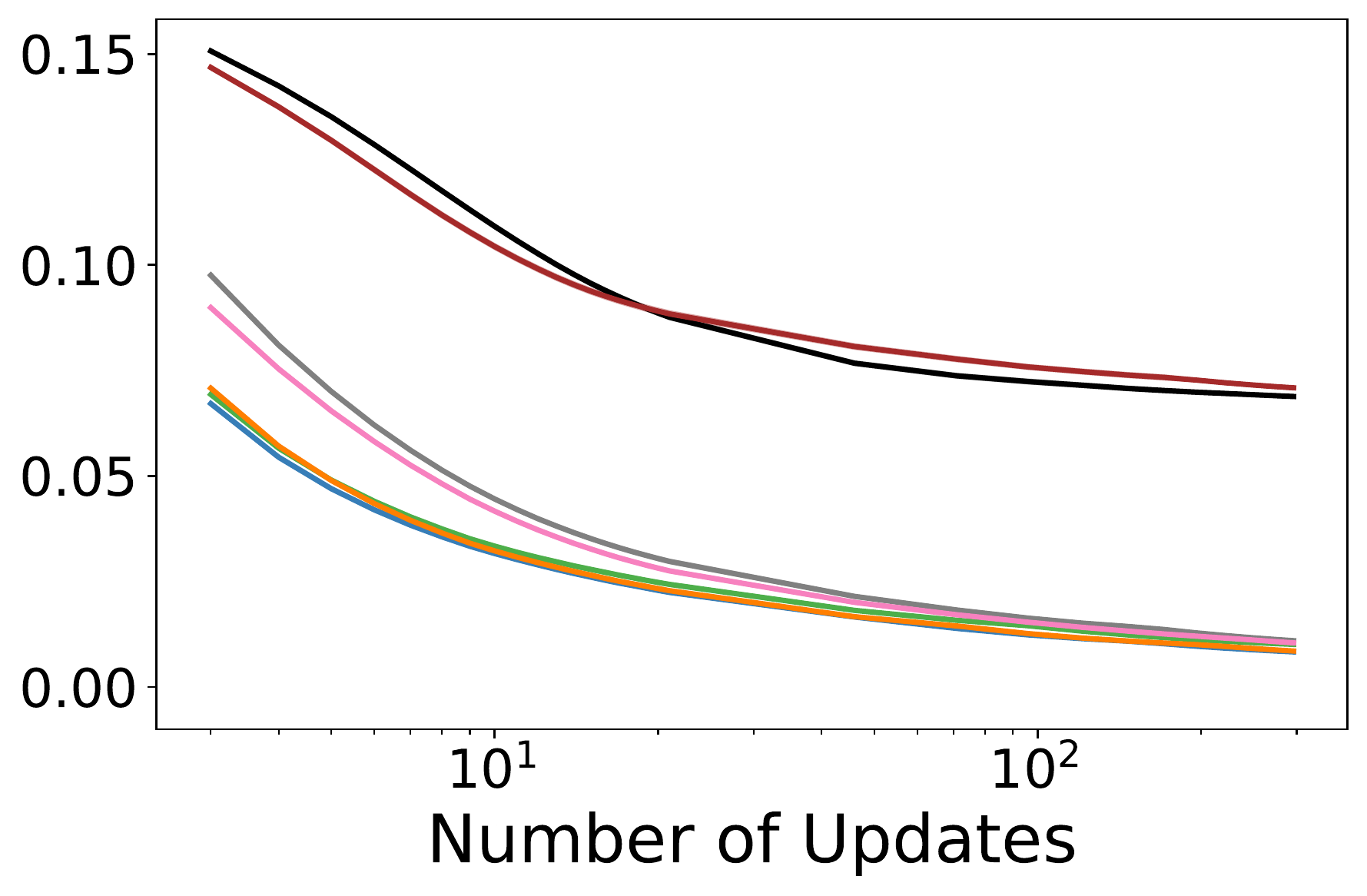}
  \end{subfigure}\hfill
  \begin{subfigure}[t]{.2\textwidth}
      % AUC
      \caption*{Boyan}
      \includegraphics[width=\textwidth]{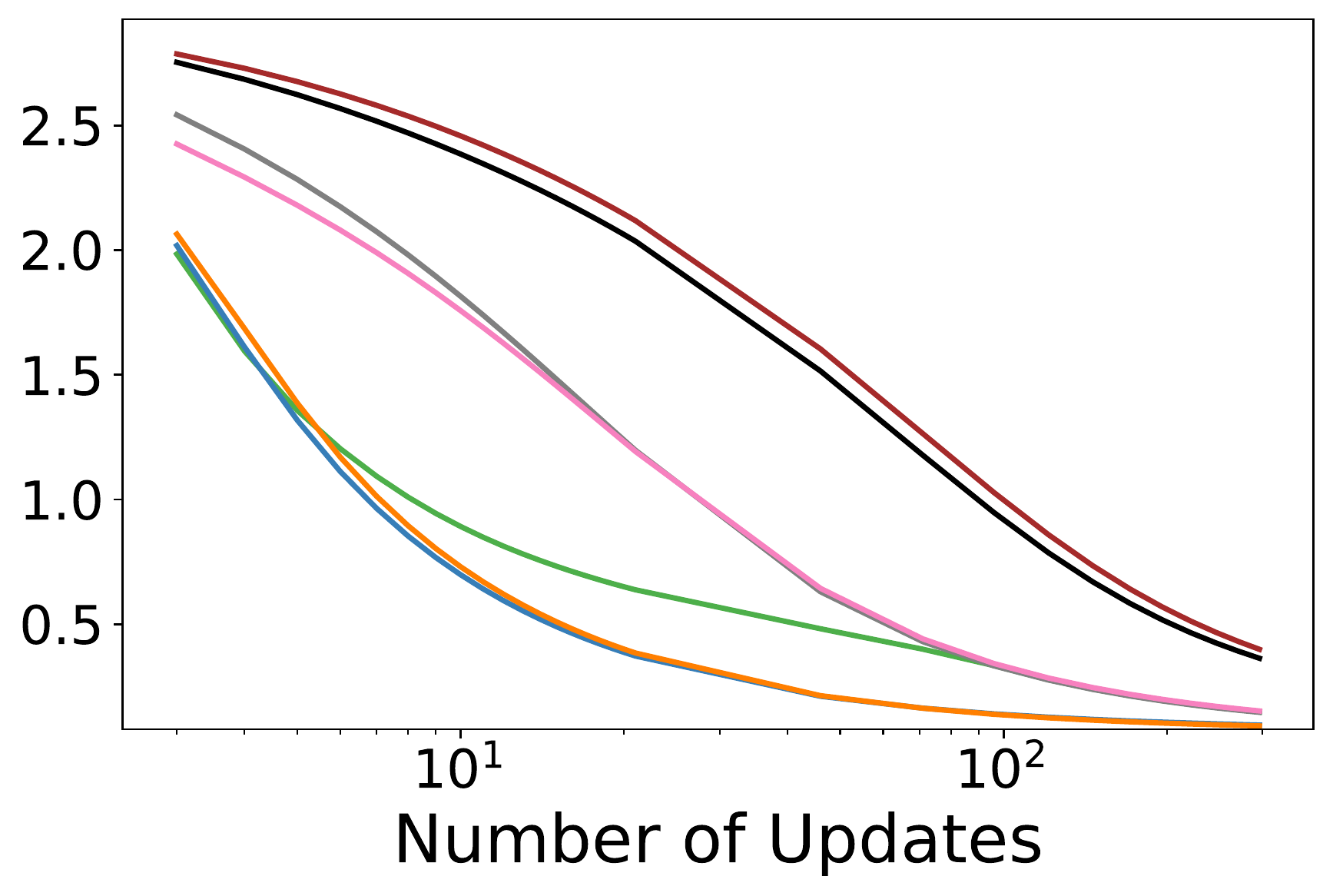}
  \end{subfigure}\hfill
  \begin{subfigure}[t]{.2\textwidth}
      % AUC
      \caption*{Baird}
      \includegraphics[width=\textwidth]{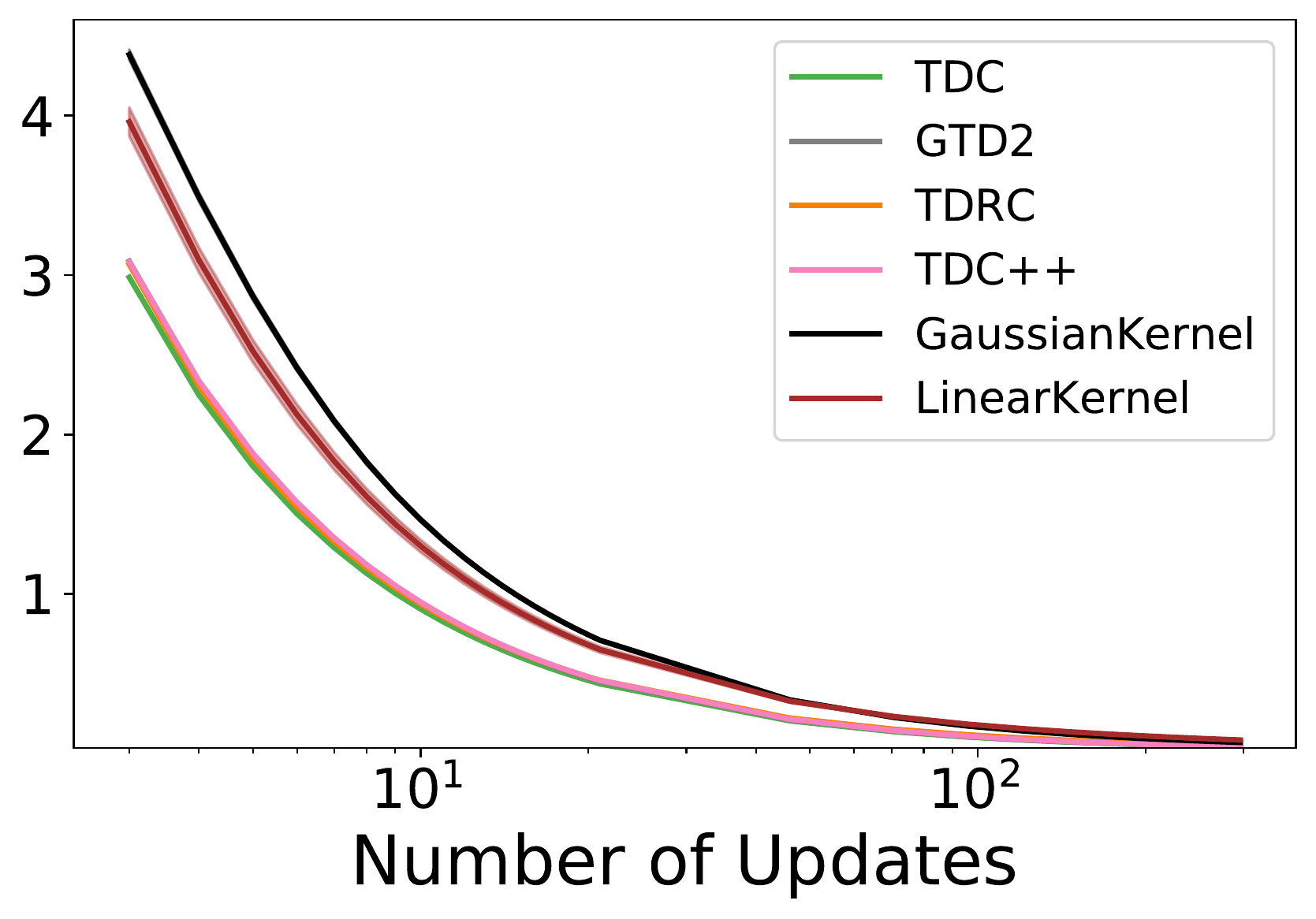}
  \end{subfigure}\hfill

  \caption{
    Sensitivity to the number of update steps for the offline batch setting.
    Each problem used a dataset of 100k samples sampled from the stationary distribution, then mini-batch updates used 8 independent samples from the dataset.
    On the x-axis we show a log-scale number of updates for each algorithm, on the y-axis we show the area under the RMSPBE learning curve averaged over 500 independent runs and 500 independently sampled datasets, with shaded regions showing the standard error over runs.
    For each number of update steps shown, we sweep over stepsizes and select the best stepsize for that number of updates; stepsizes were swept from $\alpha \in \{2^{-5}, 2^{-4}, \ldots, 2^0\}$. For TDRC, we set $\beta = 1$.
    This effectively shows the best performance of each algorithm if it was only given a fixed number of updates.
    GTD2 and the Kernel-RG methods show notably slower convergence than other methods.
  }
  \label{fig:batch_update_sensitivity}
\end{figure*}

\section{Incorporating Accelerations}\label{app_acc}
\begin{figure*}[h!]
  \centering
  \includegraphics[width=\textwidth]{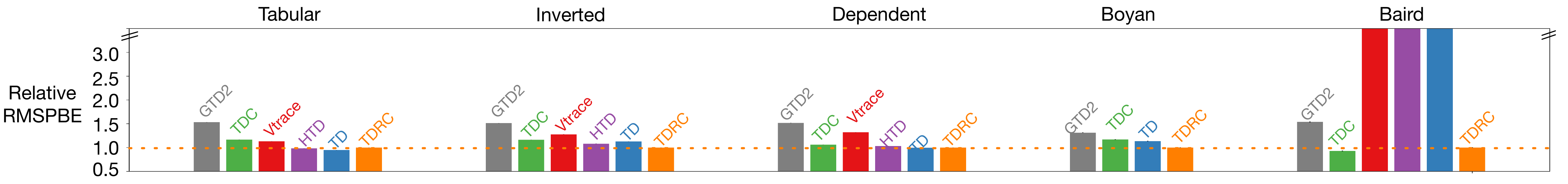}
  \caption{
    Relative performance of methods using the \textbf{Adam} stepsize selection algorithm, compared using the average area under the RMSPBE learning curve.
    Values swept are: $\alpha \in \{2^{-8}, 2^{-7}, \ldots , 2^{-1}, 2^{0} \}$ and as before, we set $\beta = 1$ for TDRC.
    On Baird's counterexample, TD, HTD, and VTrace all exhibit slow learning as well. The actual number for area under the learning curve are shown in Table~\ref{tab:adam_stepsize}.
  }
  \label{fig:amsgrad_results}
\end{figure*}

True stochastic gradient methods provide the benefit that they should be amenable to accelerations for stochastic approximation, such as momentum, mirror-prox updates (Juditsky \& Nemirovski, 2011), and variance reduction techniques (Du et al., 2017). This is in fact one of the arguments motivating GTD2, and its formulation as a saddlepoint method.

We begin investigating how acceleration in the online prediction setting impacts the overall performance and relative ordering of the algorithms. Momentum is commonly used in online deep RL systems, and is a form of acceleration. We compare all the methods using Adam (Kingma \& Ba, 2014; Reddi, Kale \& Kumar, 2019), which includes momentum. Several recently proposed optimizers include momentum and are best viewed as extensions of Adam. Here we use Adam as there is little evidence in the literature that these new variants are better than Adam for online updates. We sweep over values of the meta-parameters in Adam, $\beta_1, \beta_2 \in \{0.9, 0.99, 0.999 \}$, and select the values that best minimize the total RMSPBE separately for each algorithm.

The bar plot in Figure~\ref{fig:amsgrad_results} parallels Figure \ref{fig:alpha_sensitivity_plus_bar_plot}, which uses Adagrad, with similar conclusions. The only notable difference is that TDC's performance on Boyan's chain is much better, though it is still not as good as TD and TDRC. Overall, the use of momentum did not accelerate convergence, with performance similar to Adagrad. The comparison is not perfect, as Adagrad allows the stepsizes to decrease to zero, which enables the algorithms to converge nicely on these domains. Adam does not due to the exponential average in the squared gradient term. These results, then, mainly provide a sanity check that results under an alternative optimizer are consistent with the previous results.

The majority of accelerations that can be used in policy evaluation are designed for off-line batch updates. Although we are more concerned with online performance, we use the batch setting in Appendix~\ref{batch_setting} as a sanity check to ensure that none of the recently proposed accelerated policy evaluation methods significantly outperform TD, TDC, or TDRC.
In addition we include Kernel Residual Gradient (Kernel-RG) (Feng, Li \& Liu, 2019).
Figure~\ref{fig:batch_update_sensitivity} shows the performance of several methods given a fixed budget number of updates.
Surprisingly, the Kernel-RG methods show much slower convergence across all problems tested.

\section{Sensitivity to the Scale of $\vech$} \label{sec:reward_scale_sensitivity}
\begin{figure*}
  \centering

  \begin{subfigure}[t]{.32\textwidth}
      % AUC
      \caption*{Tabular}
      \includegraphics[width=\textwidth]{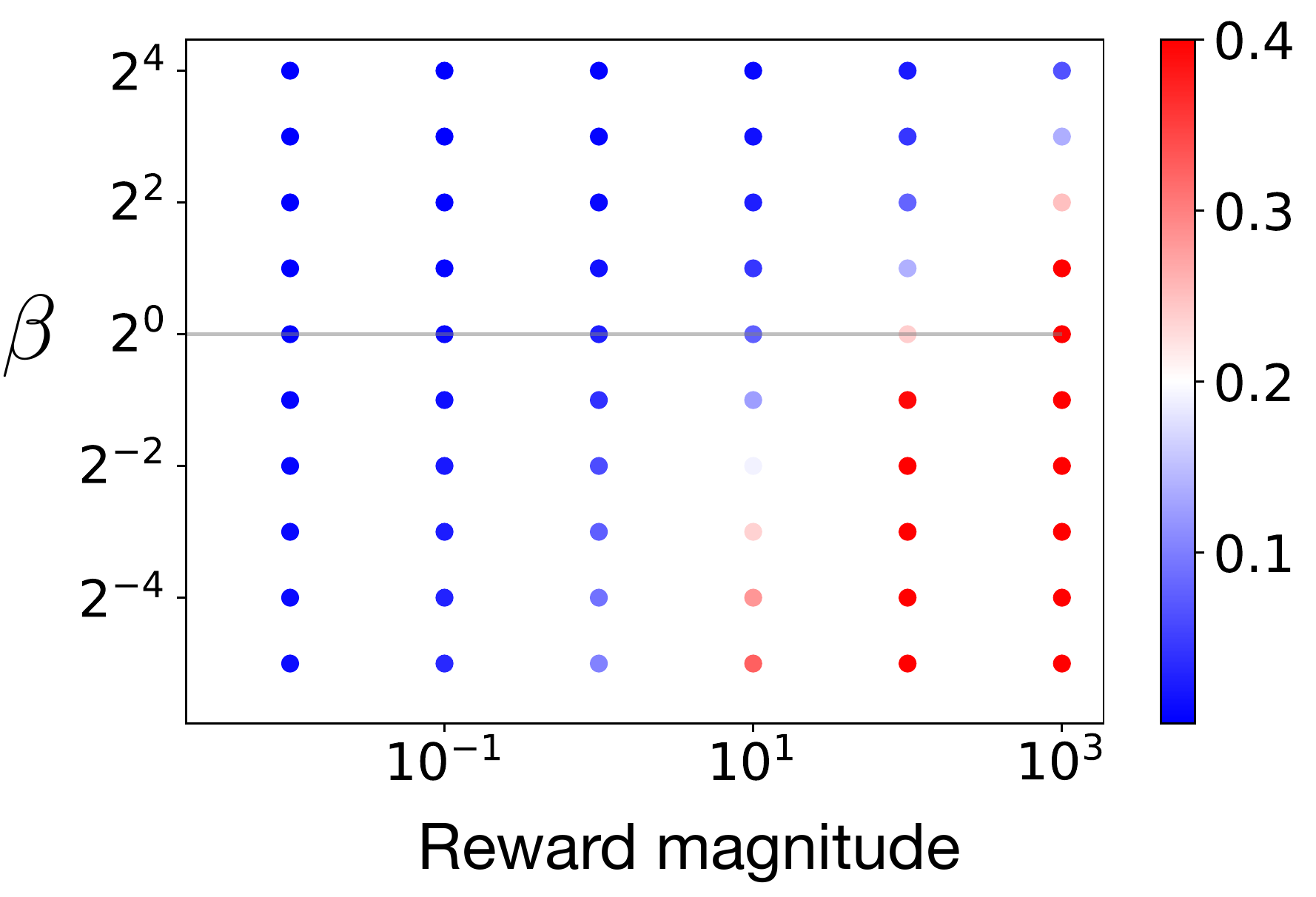}
  \end{subfigure}\hfill
  \begin{subfigure}[t]{.32\textwidth}
      % AUC
      \caption*{Inverted}
      \includegraphics[width=\textwidth]{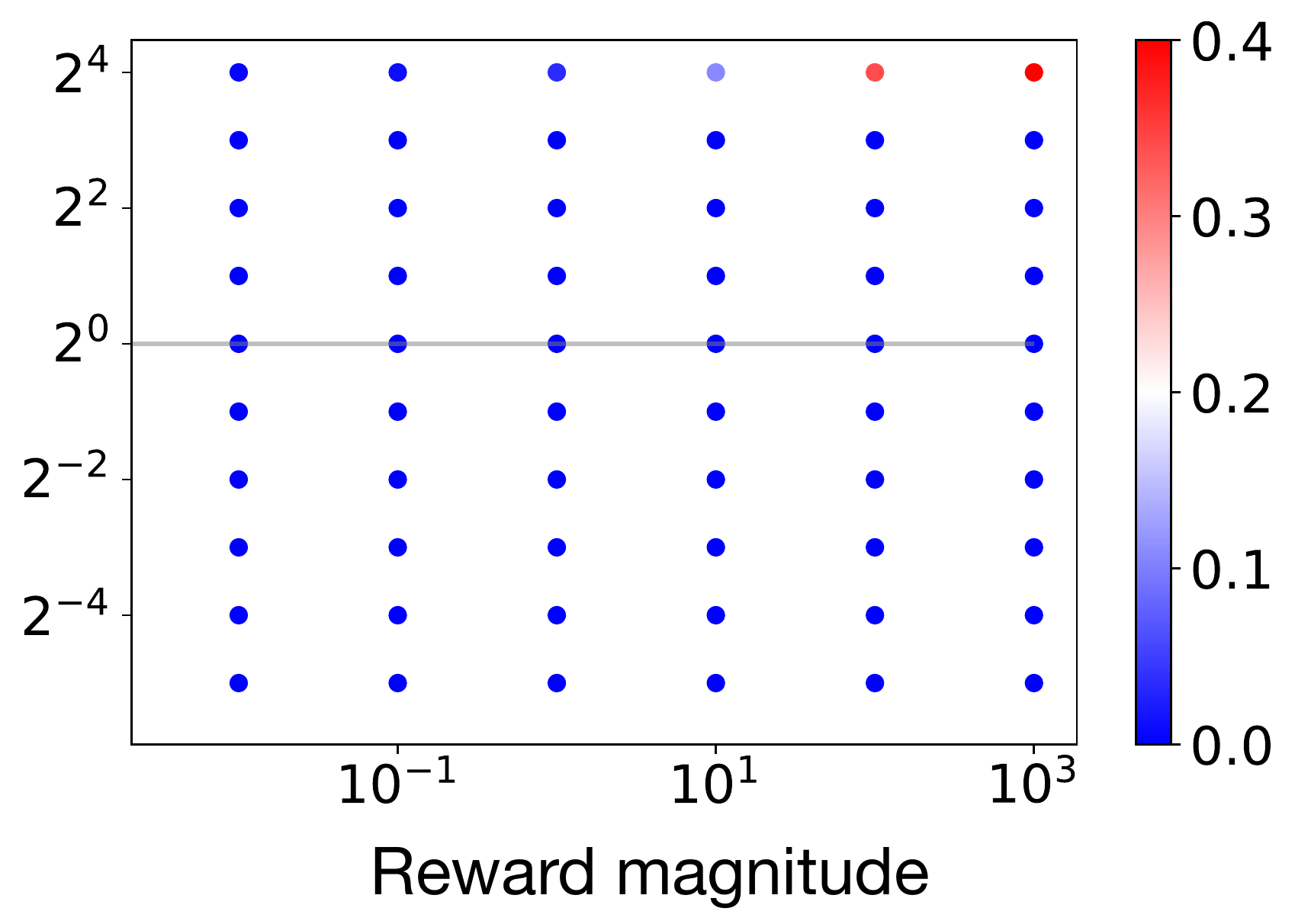}
  \end{subfigure}\hfill
  \begin{subfigure}[t]{.32\textwidth}
      % AUC
      \caption*{Dependent}
      \includegraphics[width=\textwidth]{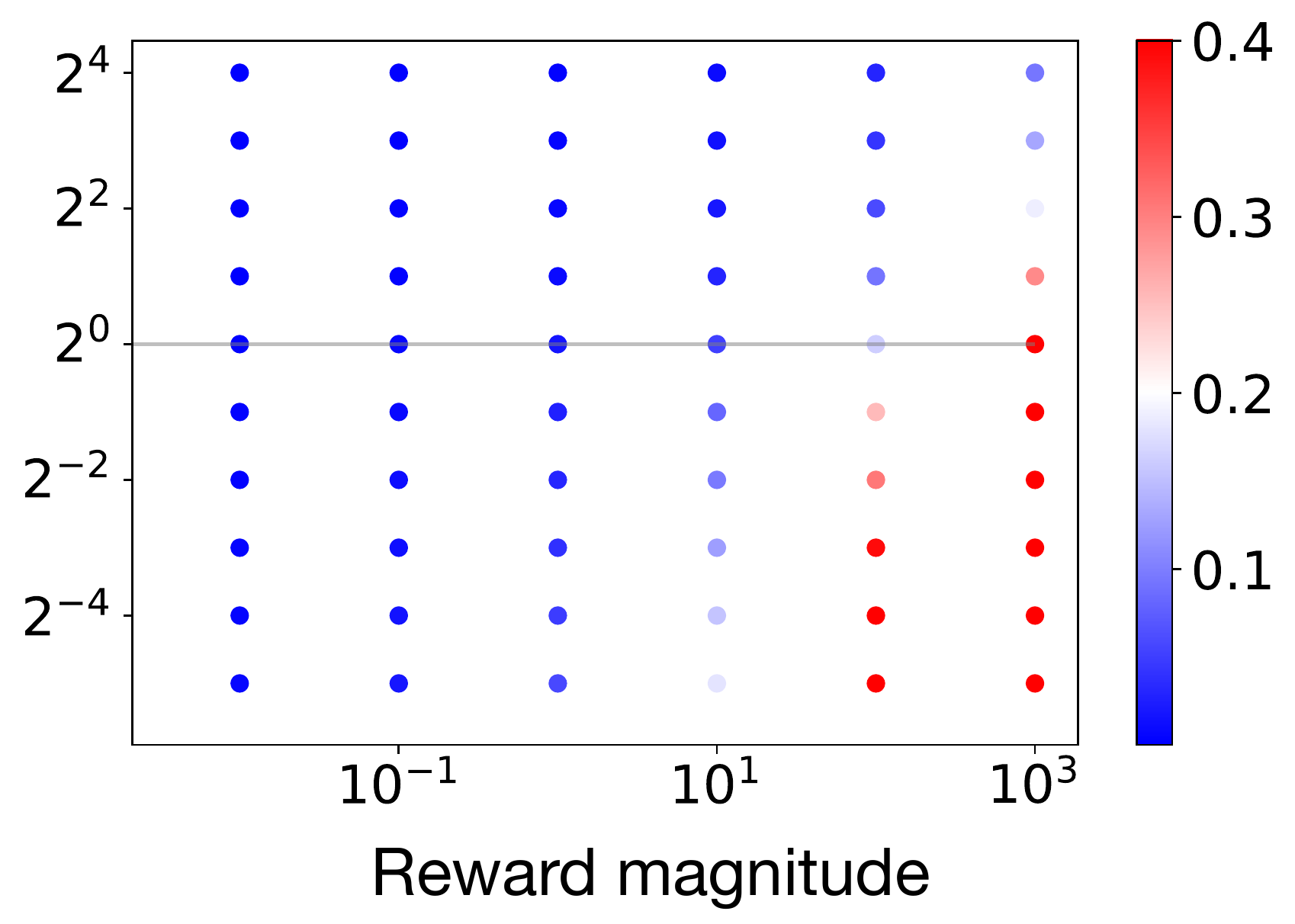}
  \end{subfigure}\hfill
  \caption{
    Relationship between TDRC and TD performance across different reward scales for different values of beta.
    On the x-axis we show the scale of the rewards for the terminal states of the random walk, on the y-axis we show a range of values of $\beta$.
    Each dot represents the number of standard deviations away from TD that TDRC's performance is across 500 independent runs for that particular value of $\beta$.
    For each dot, TDRC and TD choose the stepsize with lowest area under the RMSPBE learning curve; with stepsizes swept from $\alpha \in \{2^{-5}, 2^{-4}, \ldots, 2^{0}\}$.
    As the scale of the rewards increases (left to right on the x-axis), the variance of the secondary weights, $\vech$, also increases; effectively requiring a larger value of $\beta$.
    This figure demonstrates that TDRC with $\beta=1$ remains relatively insensitive to the scale of the rewards except in extreme cases when the variance of the rewards from transition to transition is quite large.
  }
  \label{fig:reward_scale}
\end{figure*}

In Figure~\ref{fig:beta_sensitivity_main} we demonstrate TDRC's sensitivity to the regularization weight, $\beta$, which is responsible for balancing between the loss due to the regularizer and the mean-squared error for $\vech$.
We motivate empirically that, on a set of small domains, the scale of the regularizer does not significantly affect the performance of TDRC.
However, as the scale of $\vech$ varies we likewise expect the scale of $\beta$ to vary accordingly.

We design a set of small experiments to understand how changes in the environment cause the scale of $\vech$ to change, and how that relates to the performance of TDRC across several values of $\beta$.
The scale of $\vech$ changes whenever the size of the TD error or scale of the features change.
For these experiments, we chose to increase the range of the TD error by making the initial value function $V = \zerovec$ and manipulating the magnitude of the rewards.
We run this experiment on the five state random-walk domain with each of the feature representations used in Section~\ref{sct:ExperimentalResults}, and change only the rewards in the terminal states by a multiplicative constant.
We compute the mean and standard deviation of TD's performance across 500 independent runs and compute the number of standard deviations TDRC's mean performance is from TD's mean performance.
We let the reward vary by order of magnitudes, with the multiplicative constant taking values $\{10^{-2}, 10^{-1}, \ldots, 10^{3}\}$.
For each scaling, we test multiple values of $\beta \in \{2^{-5}, 2^{-4}, \ldots, 2^4\}$ and for each of these instances we select the best constant stepsize from $\{2^{-5}, 2^{-4}, \ldots, 2^{-1}\}$.

In Figure~\ref{fig:reward_scale}, we show the range of $\beta$ for which TDRC's performance is as good, or nearly as good, as TD's performance as the magnitude of the rewards increases.
As hypothesized, the range of acceptable $\beta$ decreases as the reward magnitude increases; however, the range of $\beta$ only appreciably shrinks for a pathologically large deviation between rewards and initial value function.
This demonstrates that, while $\beta$ is problem dependent, its range of acceptable values is robust to all but the most pathological of examples across several different representations.

\section{Investigating QC on Mountain Car} \label{sec: qc_mc_non_linear_investigation}
In this section we include a deeper preliminary investigation into the performance of QC on the Mountain Car environment with non-linear function approximation.
As we observed in Figure~\ref{fig:NonlinearMCAndCP}, QC performed considerably worse than either Q-learning and QRC.
We hypothesize that this poor performance is the result of high variance updates to the value function estimate due to a poor estimate of $\CE{\delta_t}{S = s_t}$.
We relax the restrictions on the secondary stepsize, $\eta\alpha$, by using $\eta = \frac{1}{2}$, allowing QC to become more like Q-learning and reducing the variance of the update to the secondary weights.
We conclude by investigating the effects of prioritization of the replay buffer by drawing samples according to the squared TD error.

\begin{figure}[!t]
  \centering

  \includegraphics[width=0.48\textwidth]{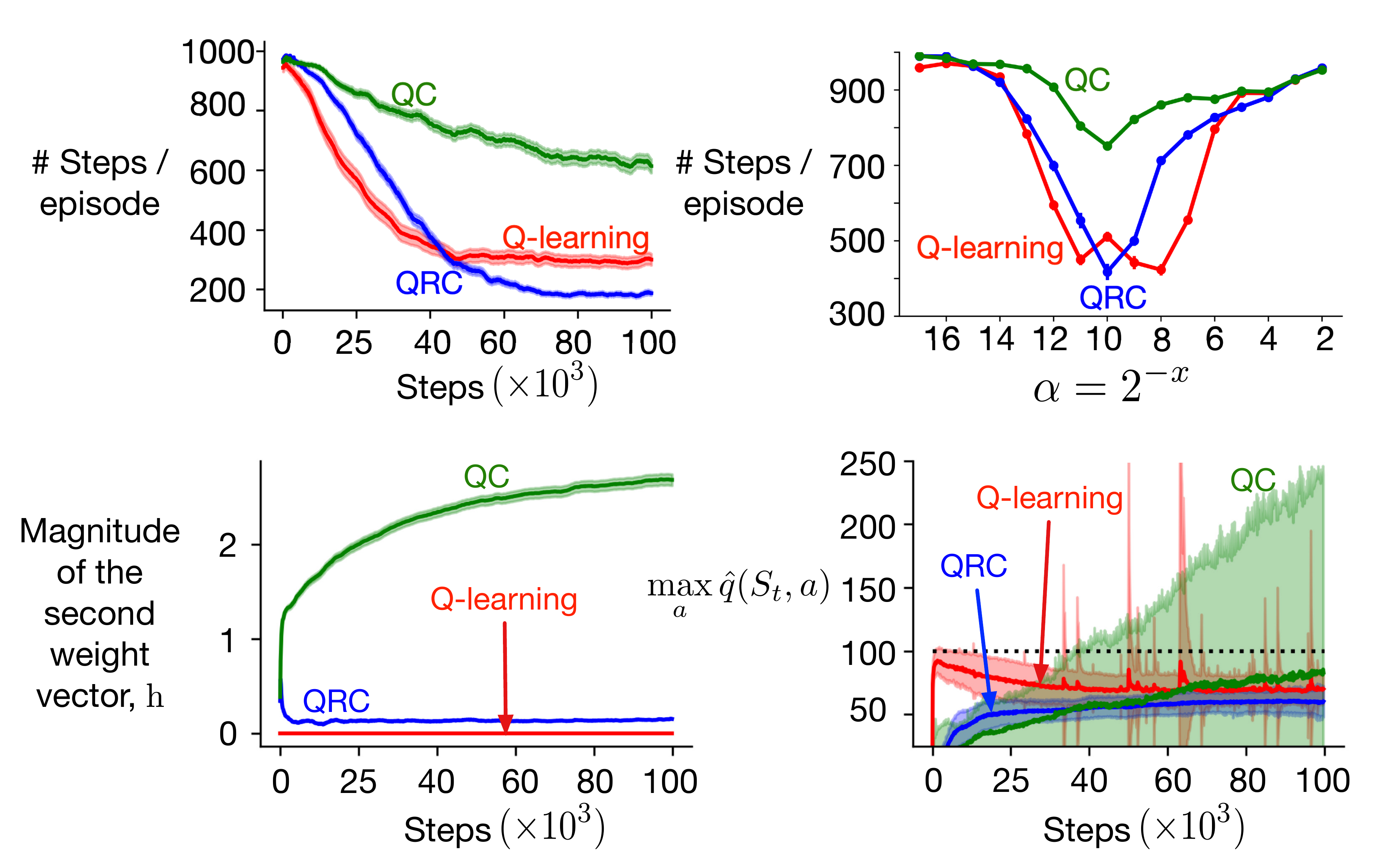}
  \caption{
    Control methods on Mountain Car with neural network function approximation.
    Each method takes one update step for every environment step and uses $\eta = 1$.
    \textbf{Top Left:} Average number of steps to goal.
    \textbf{Top Right:} Sensitivity to stepsize showing area under the learning curve for each value of $\alpha$.
    \textbf{Bottom Left:} Magnitude of the secondary weights for each algorithm. Q-learning is included as a flat line at zero, as Q-learning is effectively a special case of QRC where the secondary weights are always $\vec0$.
    \textbf{Bottom Right:} Mean and standard deviation of the maximum action-value for each step of learning. QC exhibited massive growth in action-values throughout learning and Q-learning exhibited periodic spikes of instability.
  }
  \label{fig:1replayMC}
\end{figure}

\begin{figure}[!t]
  \centering

  \includegraphics[width=0.48\textwidth]{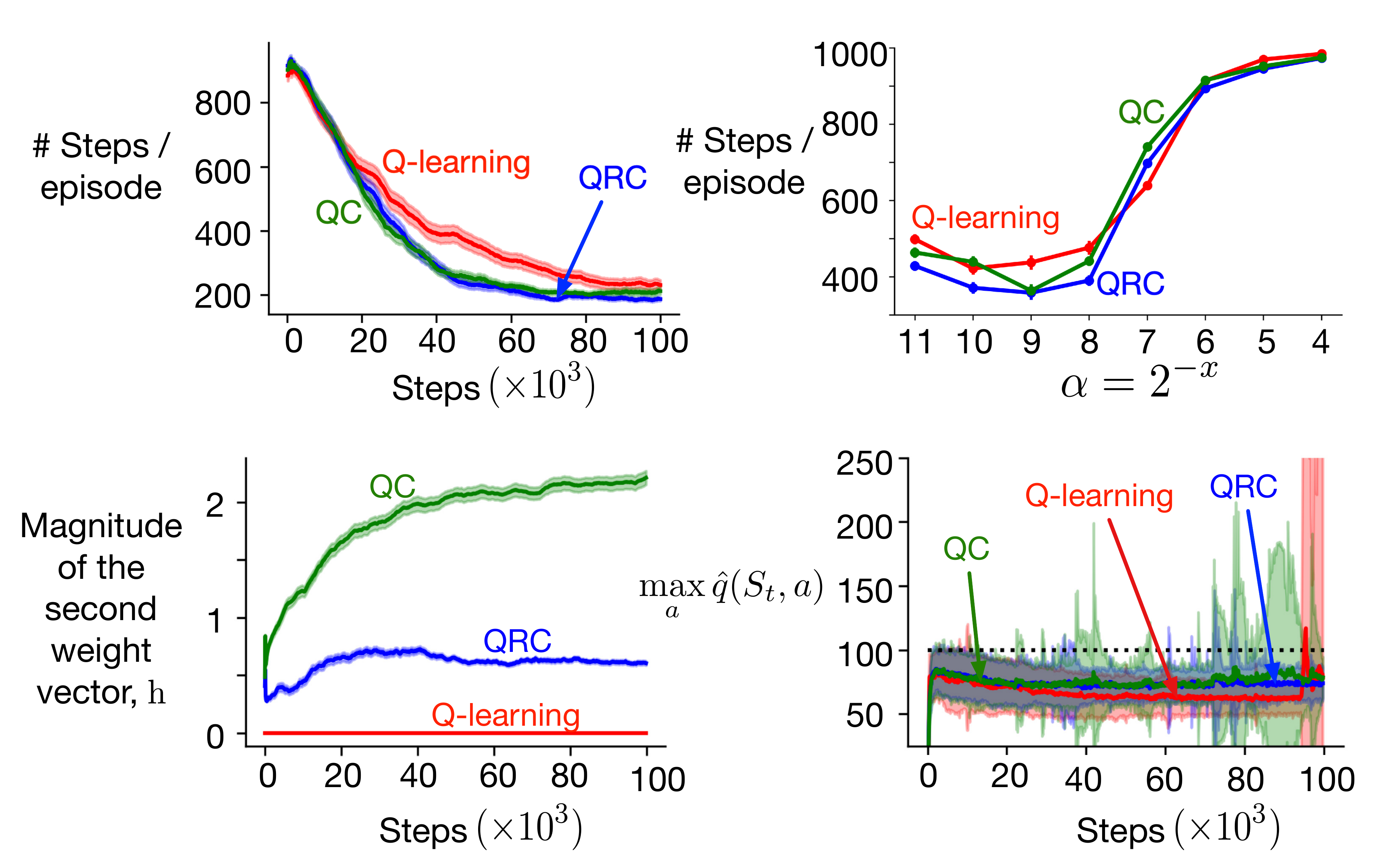}
  \caption{
    Same as Figure~\ref{fig:1replayMC} except $\eta=0.5$.
    Learning performance of QC is now competitive with Q-learning and QRC, though QC and Q-learning both exhibited more instability than QRC.
  }
  \label{fig:1replayMChalf}
\end{figure}

\begin{figure}
  \centering

  \includegraphics[width=0.48\textwidth]{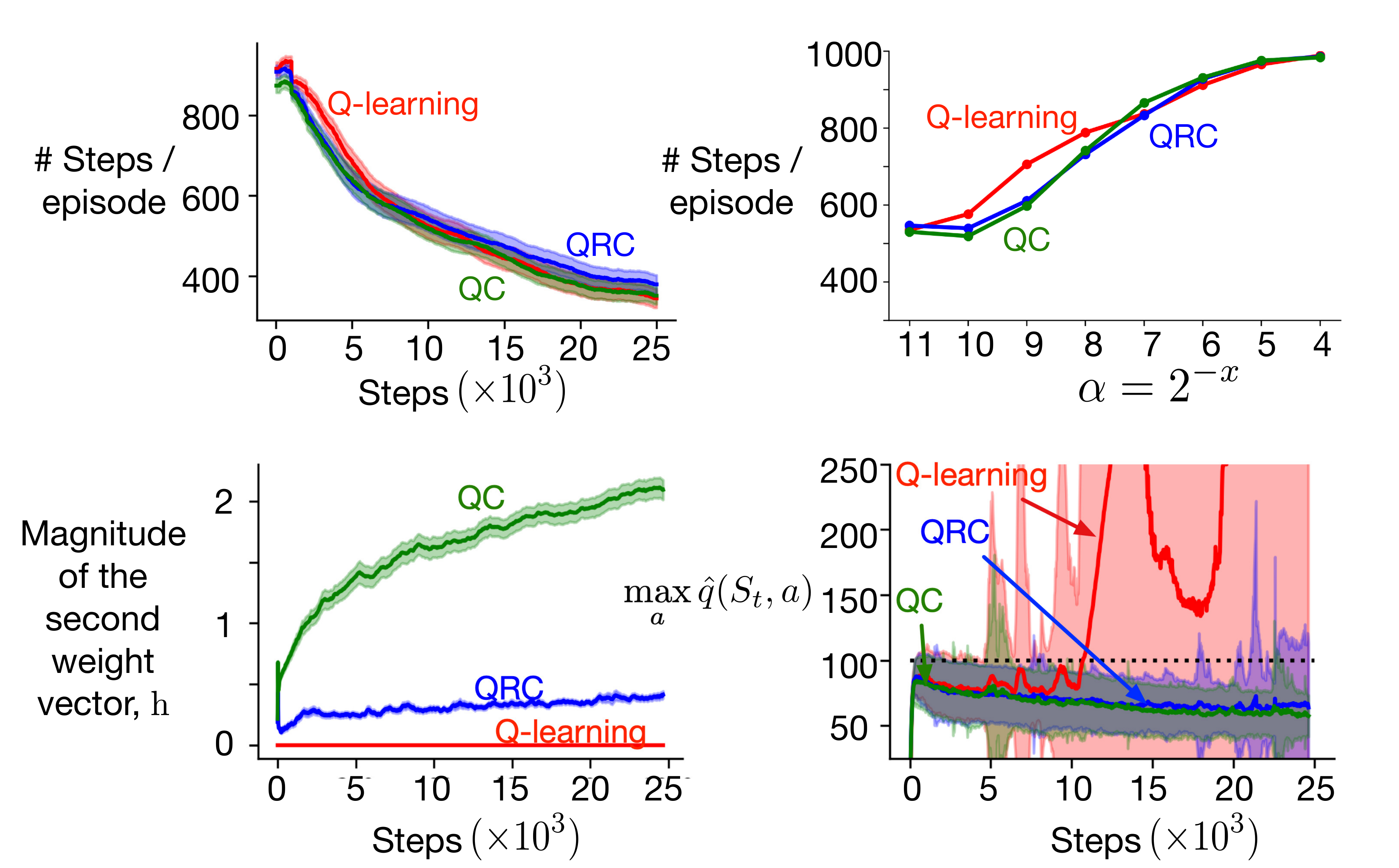}
  \caption{
    Same as Figure~\ref{fig:1replayMC} except $\eta=0.5$ and each method takes \emph{ten} update steps for every environment step using prioritized experience replay.
  }
  \label{fig:10replayMC}
\end{figure}

We start by investigating the performance of each algorithm when only a single step of replay is used on each environmental step.
The learning curve in Figure~\ref{fig:1replayMC} reaffirms that QRC and Q-learning significantly outperform QC in this setting.
Interestingly, the norm of QC's secondary set of weights grows nearly monotonically throughout learning while in contrast, QRC's secondary weights start large at the beginning of learning and quickly shrink as the value function estimates become more accurate.
The bottom right curve shows the mean and standard deviation of the maximum absolute value of $\hat{q}(S_t, \cdot)$ for each step of learning.
The variance of QC's maximum state-action value increased significantly over the maximum observable return in the Mountain Car domain---which is represented by a dashed line at 100.
These plots in combination suggest that QRC's additional constraint on the magnitude of the secondary weights helps stabilize the learning system when using neural network function approximators.

One plausible explanation for QC's poor performance is that the TD error is high variance in the Mountain Car environment, increasing the variance of the stochastic updates to the secondary weights.
We test this hypothesis by decreasing the stepsize for the secondary weights.
If the variance of the updates is large, then a smaller stepsize can help stabilize learning.
We choose $\eta = \frac{1}{2}$ and otherwise keep all other empirical settings the same.

Figure~\ref{fig:1replayMChalf} shows that QRC and QC now perform very similarly and only slightly outperform Q-learning.
As discussed in Section~\ref{sec_overall}, decreasing the secondary stepsize makes both TDC and TDRC behave more similarly to TD, so this result is not surprising.
Interestingly, Figure~\ref{fig:1replayMChalf} shows that still the magnitude of the secondary weights quickly grows for QC; however, unlike the previous experiment, the secondary weights for QRC do not quickly decay either.

Given that each of the algorithms seem to perform similarly when $\eta = \frac{1}{2}$, we revisit the highly off-policy experiment shown in Figure~\ref{fig:NonlinearMCAndCP} when $\eta = \frac{1}{2}$.
To further exaggerate the off-policy sampling, we additionally prioritize the experience replay buffer by drawing samples according to their squared TD error.
Figure~\ref{fig:10replayMC} shows that, while the learning curve performance between algorithms appears to be the same, Q-learning exhibits significant instability in its value function approximation.

\begin{figure*}
  \centering
  \includegraphics[width=\textwidth]{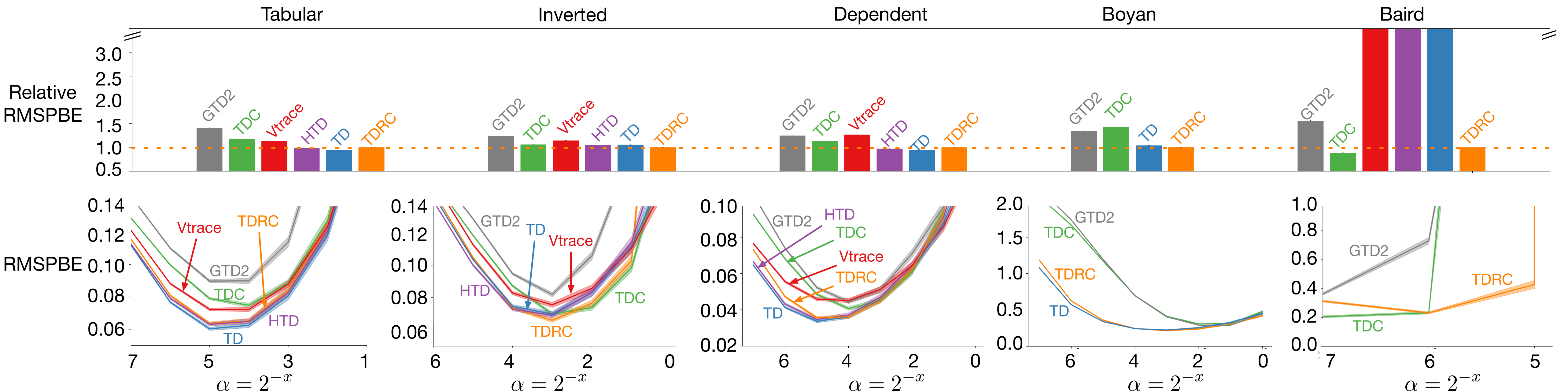}
  \caption{
    \textbf{Top:} The normalized average area under the RMSPBE learning curve for each method on each problem using a \textbf{constant} stepsize.
    Each bar is normalized by TDRC's performance so that each problem can be shown in the same range.
    All results are averaged over 200 independent runs with standard error bars shown at the top of each rectangle, though most are vanishingly small.
    \textbf{Bottom:} stepsize sensitivity measured using average area under the RMSPBE learning curve for each method on each problem.
    HTD and VTrace are not shown in Boyan's Chain because they reduce to TD for on-policy problems. The values corresponding to the bar graphs are given in Table \ref{tab:constant_stepsize}.
  }
  \label{fig:all_bar-plot_constant}
\end{figure*}

\begin{figure*}
  \centering
  \includegraphics[width=\textwidth]{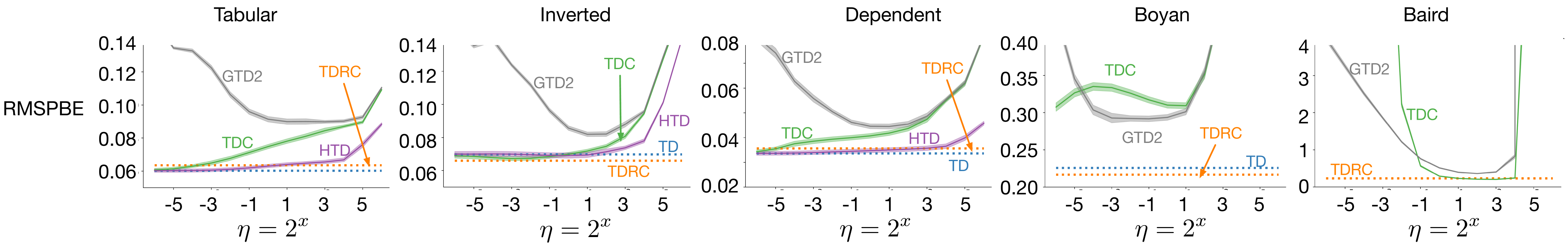}
  \caption{
    Sensitivity to the second stepsize, for changing parameter $\eta$. All methods use a constant stepsize $\alpha$. All methods are free to choose any value of $\alpha$ for each specific value of $\eta$. Methods that do not have a second stepsize are shown as flat line.
  }
  \label{fig:eta_sensitivity_constant}
\end{figure*}

These preliminary experiments suggest that, like TDC, QC's performance is highly driven by the magnitude of its secondary stepsize.
When the secondary stepsize is well-tuned QC shows similar stability to QRC; while QRC remains stable across all experimental settings.
Q-learning, like TD, is sensitive to the degree of off-policy data, becoming increasingly unstable as more off-policy updates are made.
In each of the experimental settings included in this section, Q-learning exhibited occasional spikes of instability; further motivating the desire to extend sound Gradient TD methods for non-linear control.

\section{Additional Linear Prediction Results} \label{sct:AdditionalResults}

In this section we include additional results supporting the experiments run in the main body of the text.
The primary conclusions drawn from these results were redundant with experiments in the text, but are included here for completeness.

We include results analogous to those in Section~\ref{sct:ExperimentalResults}, except using a constant stepsize on all problems.
While constant stepsizes are not commonly used in practice, they are useful for drawing clear conclusions without stepsize selection algorithm playing a confounding role.
We show in Figure~\ref{fig:all_bar-plot_constant}, that the relative performance between methods does not change when using a constant stepsize.
We do notice that TDC performs more similarly to HTD, TD, and TDRC in the constant stepsize case, which suggests that TDC benefits less from using Adagrad than these other methods.

Figure~\ref{fig:eta_sensitivity_constant} shows that algorithms are generally more similar in terms of stepsize sensitivity.
This suggests that differences in between the algorithms are less pronounced when using constant stepsizes, which provides more support for the argument that empirical comparisons should simultaneously consider modern stepsize selection algorithms.

For completeness, we include the values visualized in Figure~\ref{fig:alpha_sensitivity_plus_bar_plot} as a table of values in Table~\ref{tab:adagrad_stepsize}.
The standard error is reported for each entry in the table.
The bold entries highlight the algorithm with the lowest RMSPBE for the given problem.
The same is included for Figure~\ref{fig:amsgrad_results} in Table~\ref{tab:adam_stepsize} and for Figure~\ref{fig:all_bar-plot_constant} in Table~\ref{tab:constant_stepsize}.

\begin{table*}
  \begin{adjustwidth}{-1in}{-1in}
    \begin{center}
      \begin{tabular}{c||ccccc}
        &Tabular & Inverted & Dependent & Boyan & Baird\\
        \hline \hline
        GTD2 & 0.079 $\pm$ 0.001 & 0.063 $\pm$ 0.001 & 0.041 $\pm$ 0.001 & 0.269 $\pm$ 0.003 & 0.357 $\pm$ 0.009\\
        \hline
        TDC & 0.063 $\pm$ 0.001 & 0.053 $\pm$ 0.001 & 0.034 $\pm$ 0.001 & 0.639 $\pm$ 0.001 & \textbf{0.196 $\pm$ 0.007}\\
        \hline
        HTD & 0.048 $\pm$ 0.001 & 0.048 $\pm$ 0.001 & 0.025 $\pm$ 0.001 & -- & 2.123 $\pm$ 0.013\\
        \hline
        TD & \textbf{0.046 $\pm$ 0.001} & 0.051 $\pm$ 0.001 & \textbf{0.024 $\pm$ 0.001} & 0.248 $\pm$ 0.003 & 4.101 $\pm$ 0.095\\
        \hline
        VTrace & 0.060 $\pm$ 0.001 & 0.059 $\pm$ 0.001 & 0.038 $\pm$ 0.001 & -- & 4.101 $\pm$ 0.095\\
        \hline
        TDRC & 0.049 $\pm$ 0.001 & \textbf{0.047 $\pm$ 0.001} & 0.026 $\pm$ 0.001 & \textbf{0.222 $\pm$ 0.002} & 0.242 $\pm$ 0.006\\
      \end{tabular}
    \end{center}
  \end{adjustwidth}
  \caption{
    \label{tab:adagrad_stepsize}
    Average area under the RMSPBE learning curve for each problem using the \textbf{Adagrad} algorithm.
    Bolded values highlight the lowest RMSPBE obtained for a given problem.
    % All TD, HTD, and VTrace appear to converge very slowly with Adagrad.
    % HTD still exhibits oscillating behavior and TD and VTrace show significant bias in final performance.
    These values correspond to the bar graphs in Figure \ref{fig:alpha_sensitivity_plus_bar_plot}.
  }
\end{table*}

\begin{table*}
  \begin{adjustwidth}{-1in}{-1in}
    \begin{center}
      \begin{tabular}{c||ccccc}
        &Tabular & Inverted & Dependent & Boyan & Baird\\
        \hline \hline
        GTD2 & 0.094 $\pm$ 0.001 & 0.074 $\pm$ 0.001 & 0.048 $\pm$ 0.001 & 0.274 $\pm$ 0.006 & 0.356 $\pm$ 0.009\\
        \hline
        TDC & 0.071 $\pm$ 0.002 & 0.057 $\pm$ 0.001 & 0.033 $\pm$ 0.001 & 0.244 $\pm$ 0.005 & \textbf{0.215 $\pm$ 0.007}\\
        \hline
        HTD & 0.060 $\pm$ 0.002 & 0.053 $\pm$ 0.001 & 0.032 $\pm$ 0.001 & -- & 3.623 $\pm$ 0.027\\
        \hline
        TD & \textbf{0.058 $\pm$ 0.002} & 0.055 $\pm$ 0.001 & \textbf{0.031 $\pm$ 0.001} & 0.237 $\pm$ 0.006 & 3.993 $\pm$ 0.053\\
        \hline
        VTrace & 0.069 $\pm$ 0.001 & 0.063 $\pm$ 0.001 & 0.042 $\pm$ 0.001 & -- & 3.993 $\pm$ 0.053\\
        \hline
        TDRC & 0.061 $\pm$ 0.001 & \textbf{0.049 $\pm$ 0.001} & 0.031 $\pm$ 0.001 & \textbf{0.209 $\pm$ 0.004} & 0.232 $\pm$ 0.007\\
      \end{tabular}
    \end{center}
  \end{adjustwidth}
  \caption{
    \label{tab:adam_stepsize}
    Average area under the RMSPBE learning curve for each problem using the \textbf{Adam} stepsize selection algorithm.
    Bolded values highlight the lowest RMSPBE obtained for a given problem.
    % All TD, HTD, and VTrace appear to be able to converge while using Adam, though convergence is very slow and not monotonic. 
    These values correspond to the bar graphs in Figure \ref{fig:amsgrad_results}.
  }
\end{table*}

\begin{table*}
  \begin{adjustwidth}{-1in}{-1in}
    \begin{center}
      \begin{tabular}{c||ccccc}
        &Tabular & Inverted & Dependent & Boyan & Baird\\
        \hline \hline
        GTD2 & 0.090 $\pm$ 0.001 & 0.082 $\pm$ 0.001 & 0.044 $\pm$ 0.001 & 0.292 $\pm$ 0.004 & 0.361 $\pm$ 0.009\\
        \hline
        TDC & 0.075 $\pm$ 0.001 & 0.070 $\pm$ 0.001 & 0.041 $\pm$ 0.001 & 0.309 $\pm$ 0.004 & \textbf{0.205 $\pm$ 0.007}\\
        \hline
        HTD & 0.063 $\pm$ 0.001 & 0.069 $\pm$ 0.002 & 0.035 $\pm$ 0.001 & -- & 1184.368 $\pm$ 69.421\\
        \hline
        TD & \textbf{0.060 $\pm$ 0.001} & 0.070 $\pm$ 0.002 & \textbf{0.034 $\pm$ 0.001} & 0.226 $\pm$ 0.005 & 11401.550 $\pm$ 270.628\\
        \hline
        VTrace & 0.072 $\pm$ 0.001 & 0.076 $\pm$ 0.002 & 0.045 $\pm$ 0.001 & -- & 18.239 $\pm$ 0.046\\
        \hline
        TDRC & 0.064 $\pm$ 0.001 & \textbf{0.066 $\pm$ 0.001} & 0.036 $\pm$ 0.001 & \textbf{0.217 $\pm$ 0.004} & 0.232 $\pm$ 0.006\\
      \end{tabular}
    \end{center}
  \end{adjustwidth}
  \caption{
    \label{tab:constant_stepsize}
    Average area under the RMSPBE learning curve for each problem using the a \textbf{constant} stepsize.
    Bolded values highlight the lowest RMSPBE obtained for a given problem. These values correspond to the bar graphs in Figure \ref{fig:all_bar-plot_constant}.
  }
\end{table*}

\section{Investigating Target Networks}\label{app:TargetNets}
One motivation for designing more stable off-policy algorithms is to improve learning interactions with neural network function approximators.
A currently pervasive technique for improving stability of off-policy learning with neural networks is to use target networks.
In this section, we investigate the impact of using target networks for each of the non-linear control algorithms investigated in this work.

In Figures~\ref{fig:TargetNet4}, \ref{fig:TargetNet64}, and \ref{fig:TargetNet256} we investigate the impact of synchronizing the target network to the value function approximation after every 4, 64, and 256 updates respectively.
All the experimental settings remain the same, other than the rate of target network synchronization. 
The conclusions drawn in the main body of the paper continue to hold when using target networks; QC learns very slowly which is exaggerated by increasing delay in updates to the bootstrapped target, QRC is stable and insensitive to choice of stepsize, and Q-learning performs well but is negatively impacted by the introduction of target networks on these domains.

\begin{figure}[!t]
  \centering

  \includegraphics[width=0.48\textwidth]{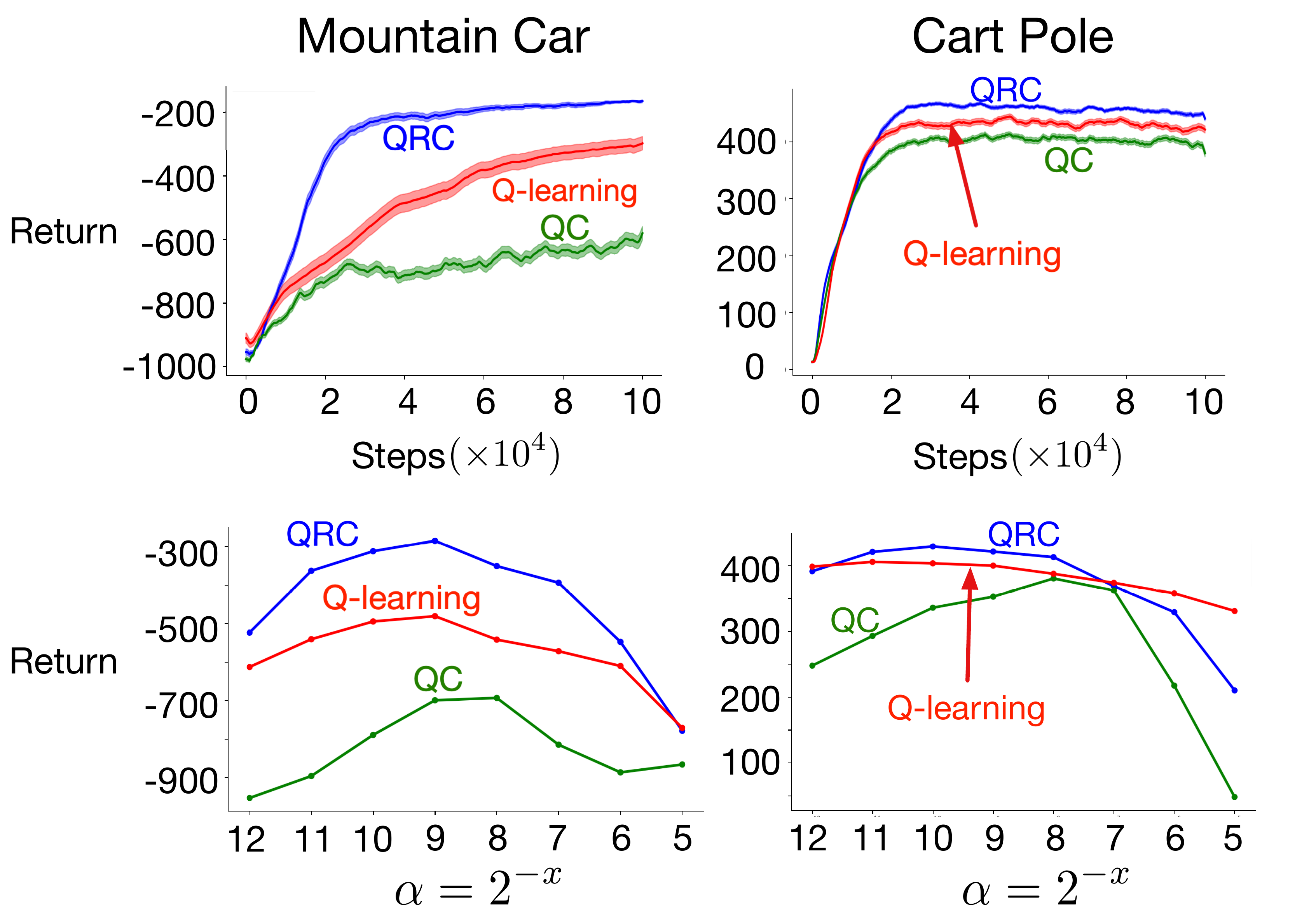}
  \caption{
    Non-linear control methods with target networks.
    Target network is synchronized with the value function after every 4 updates.
  }
  \label{fig:TargetNet4}
\end{figure}
\begin{figure}[!t]
  \centering

  \includegraphics[width=0.48\textwidth]{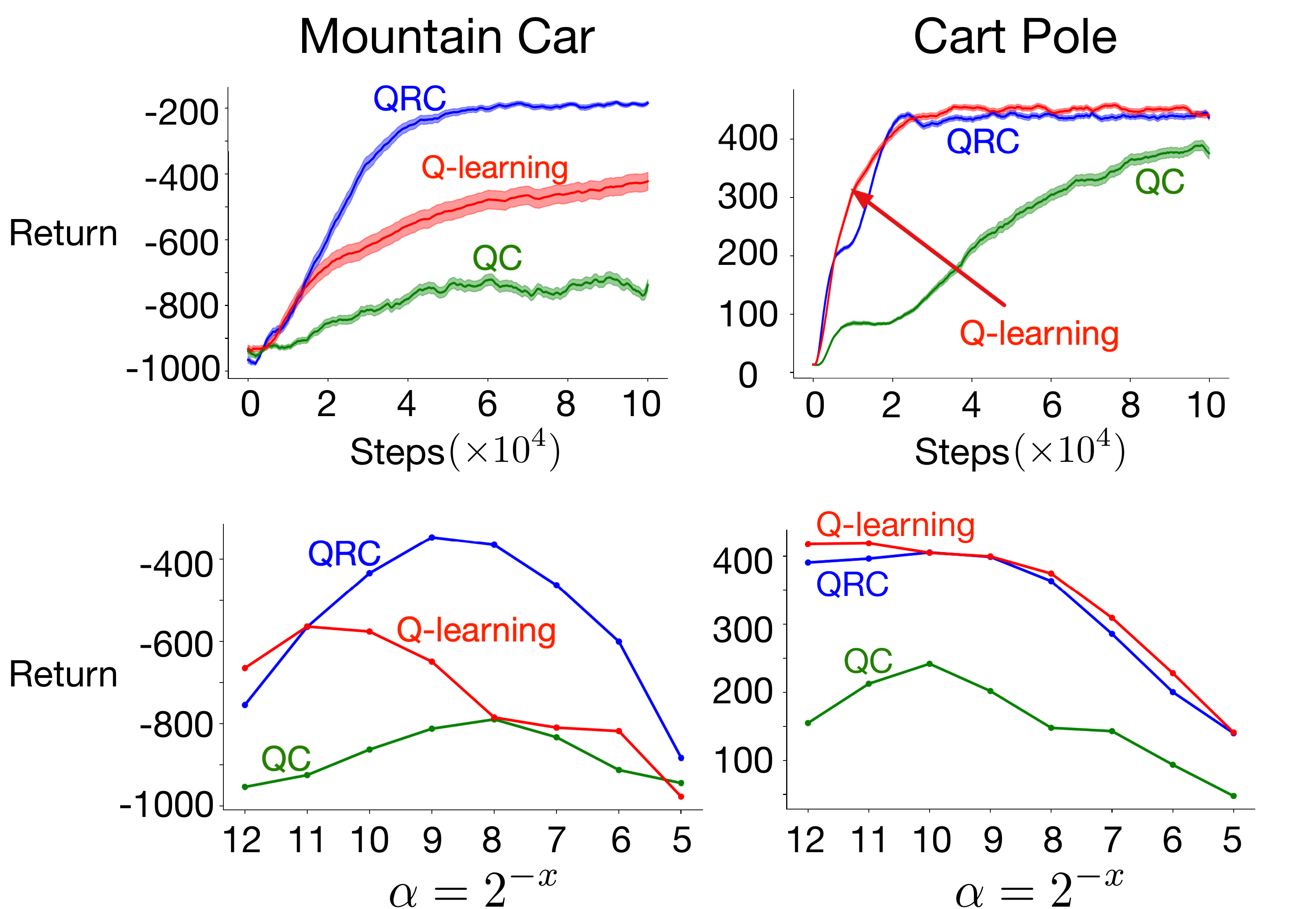}
  \caption{
    Same as Figure~\ref{fig:TargetNet4}, except target network is synchronized after every 64 updates.
  }
  \label{fig:TargetNet64}
\end{figure}
\begin{figure}[!t]
  \centering

  \includegraphics[width=0.48\textwidth]{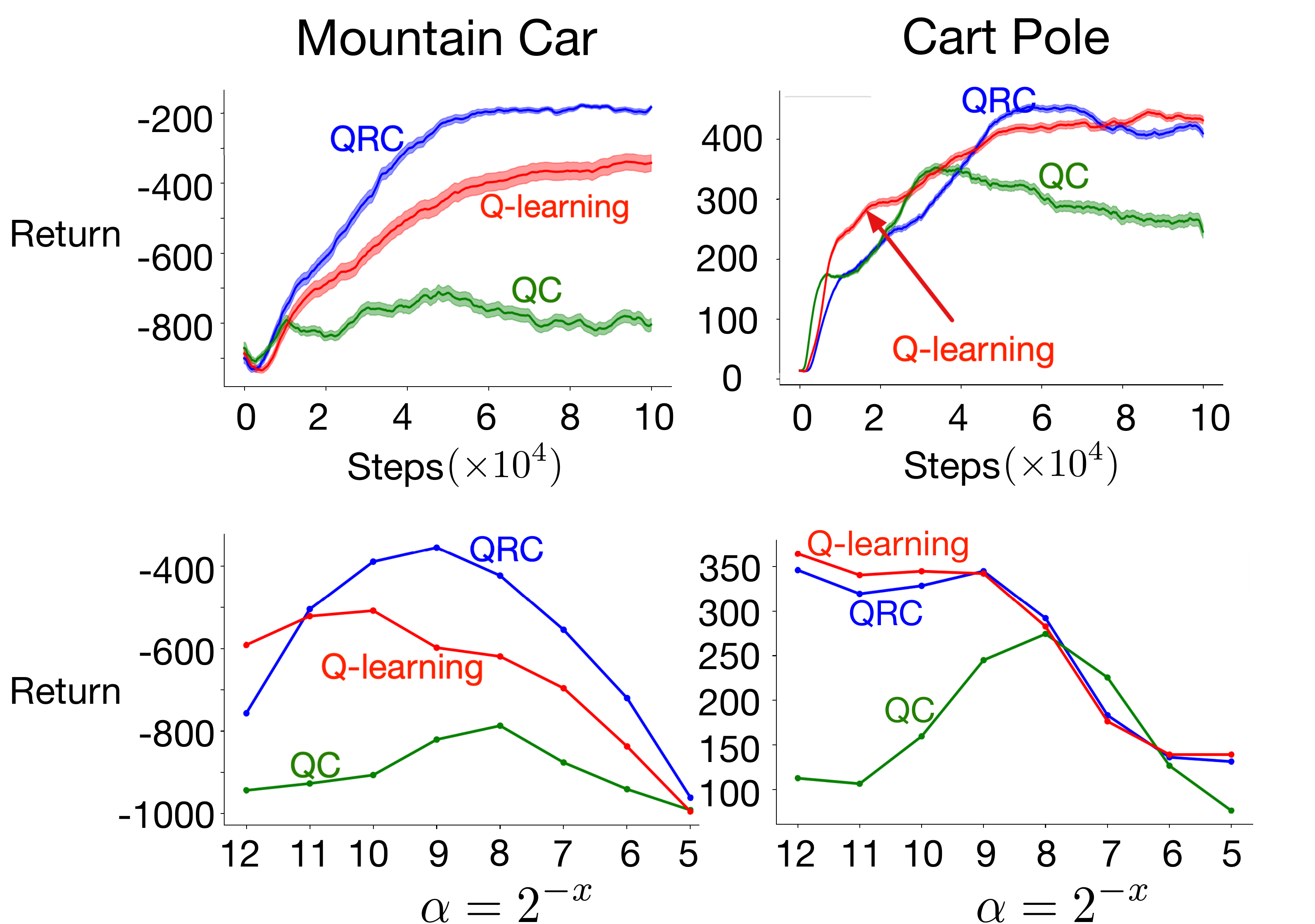}
  \caption{
    Same as Figure~\ref{fig:TargetNet4}, except target network is synchronized after every 256 updates.
  }
  \label{fig:TargetNet256}
\end{figure}

\section{Parameter Settings and Other Experiment Details}
\label{Parameter-settings}

\begin{figure}[ht]
  \centering
  \includegraphics[width=\columnwidth]{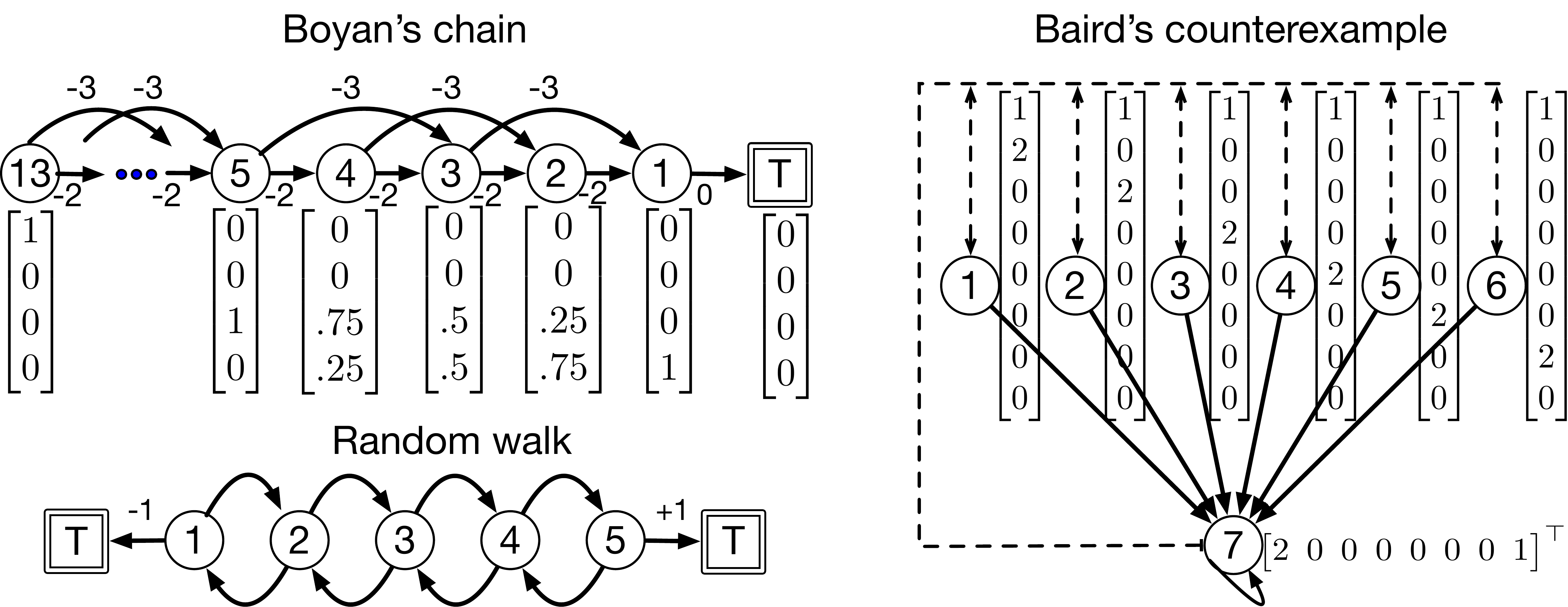}
  \caption{
    Above we provide a graphic depiction of each of the three MDPs and the corresponding feature representations used in our experiments. We omit the three feature representations used in the Random Walk due to space restrictions (see Sutton et al., 2009). All unlabeled transitions emit a reward of zero.
  }
  \label{fig:envs}
\end{figure}

\subsection{Actor-Critic Algorithm with TDRC}\label{app_ac_tdrc}
We assume that the agent's policy $\pi_{\vectheta}(A|S)$ is parameterized by weight vector $\vectheta$.
To incorporate TDRC into the one-step actor-critic algorithm (Sutton \& Barto, 2018), we simply change the update rule for the value function approximation step for the TDRC update.
This yields the following update equations for Actor-Critic with TDRC:
\begin{align*}
  \delta_t &= R_{t+1} + \gamma \vecw_t\tr \vecx_{t+1} - \vecw_t\tr \vecx_t \\
  \vecw_{t+1} &\leftarrow \vecw_t + \alpha \delta_t \vecx_t - \gamma (\vech_t\tr \vecx_t) \vecx_{t+1} \\
  \vech_{t+1} &\leftarrow \vech_t + \eta \alpha \left(\delta_t - \vech_t\tr \vecx_t \right) \vecx_t - \eta\alpha\beta \vech_t\\
  \vectheta_{t+1} &\leftarrow \vectheta_t + \alpha \gamma^{t+1} \delta_t \nabla_{\vectheta_t} \ln \pi_{\vectheta}(A_t \;|\; S_t),
\end{align*}
where the original actor-critic algorithm can be recovered with $\vech_0 = \zerovec$ and $\eta = 0$ and a TDC-based actor-critic algorithm can be obtained with $\beta = 0$.
In practice,  the $\gamma^{t+1}$ term in the update for $\vectheta$ is often dropped so, as such, in our actor-critic experiment we likewise did not include this term in our implementation. For ADAM optimizer we used $\beta_1=0.9$ and $\beta_2=0.999$. We swept over $\alpha \in \{2^{-8}, 2^{-7}, \ldots, 2^{-2}, 2^{-1} \}$ and had $\eta=1$ for TDC. We used tile coding with 5 tilings and $4\times4$ tiles.

\subsection{Prediction Experimental Details}
For the results shown in the main body of the paper on the random walk, Boyan's Chain, and Baird's Counterexample we swept over free meta-parameters for every method comparing the meta-parameters which performed best according to the area under the RMSPBE learning curve.
The stepsizes swept for all algorithms were $\alpha \in \{2^{-7}, 2^{-6}, \ldots, 2^{0}\}$.
For TDC and HTD, we swept values of the second stepsize by sweeping over a multiplicative constant times the primary stepsize, $\eta \in \{2^0, 2^1, \ldots, 2^6\}$ maintaining the convergence guarantees of the two-timescale proof of convergence for TDC.
For GTD2, we swept values of $\eta \in \{2^{-6}, 2^{-5}, \ldots, 2^{5}, 2^{6}\}$ as the saddlepoint formulation of GTD2 allows for a much broader range of $\eta$ while still maintaining convergence.

\subsection{Cart Pole and Mountain Car Experimental Details}\label{app:CPAndMCExpDetails}
To solve these task we used a fully connected neural network with two hidden layers where each layer had 64 nodes in Cart Pole (32 nodes in Mountain Car) with ReLU as the non--linearity and the output layer as linear. The weights were updated using a replay buffer of size 4,096 in Cart Pole (size 4000 in Mountain Car) and mini-batch size of 32 using ADAM optimizer with $\epsilon=10^{-8}$, $\beta_1=0.9$, and $\beta_2=0.999$. We also used ADAM optimizer for updating the $\vech$ vector using $\epsilon=10^{-8}$, $\beta_1=0.99$, and $\beta_2=0.999$. The neural network weights were initialized using Xavier initialization (Glorot \& Bengio, 2010) and the biases were initialized with a normal distribution with mean $0$ and standard deviation $0.1$. The second weight vectors were initialized to $\vec0$. Actions were selected using an $\epsilon$-greedy policy where $\epsilon=0.1$. We tested several values of the stepsize: $\{2^{-13}, ..., 2^{-2}\}$ for Cart Pole and $\{2^{-17}, ..., 2^{-2}\}$ for Mountain Car. The final results show the performance averaged over 200 independent runs. In these task we set $\eta=1$ for QC and QRC methods and set the regularization parameter $\beta=1$ for QRC.

\subsection{MinAtar Experimental Details}\label{app:MinAtarExpDetails}
We ran the MinAtar experiments for 5 million steps. Discount factor parameter, $\gamma$ was set to 0.99. The rewards were scaled by ($R \times (1-\gamma)$) so that the neural network does not have to estimate large returns. The Q-Learning and QRC network architectures were the same as that used by (Young \& Tian, 2019). The network had one convolutional layer and one fully connected layer after that. The convolutional layer used sixteen $3 \times 3$ convolutions with stride 1. The fully connected layer had 128 units. Both convolutional and fully connected layers used ReLU gates. The network is initialized the same way as (Young \& Tian, 2019). We did not use target networks for MinAtar experiments because (Young \& Tian, 2019) showed that using target networks has negligible effects on the results.

We used a circular replay buffer of size 100,000. The agent started learning when the replay buffer had 5,000 samples in it. We annealed epsilon from 1.0 to 0.1 through the first 100,000 steps and then kept it at 0.1 for the rest of the steps. The agent had one training step using a mini-batch of size 32 per environment step. As explained by (Young \& Tian, 2019), frame skipping was not necessary since the frames of the MinAtar environment are more information rich. Other hyperparameters were chosen the same as (Young \& Tian, 2019) and (Mnih et al., 2015). We used the RMSProp optimizer with a smoothing constant of 0.95, and $\epsilon=0.01$. For QRC, we used RMSProp to learn the second weight vector $\vech$.  We swept over RMSprop stepsizes in powers of 2, $\{2^{-10}, ..., 2^{-5}\}$ for breakout, and $\{2^{-12}, ..., 2^{-8}\}$ for space invaders. $\eta$ was set to 1 for QC and QRC and $\beta$ was 1 for QRC.

For the learning curve, we plotted the setting that resulted in the best area under the learning curve. We computed the moving average of returns over 100 episodes (shown in Figure~\ref{fig:breakout}) similar to (Young \& Tian, 2019). For computing the total discounted reward, we simply averaged over all of the returns that the agent got during 5 million steps to get a single number for each run and each parameter setting. We then averaged this number over 30 independent runs of the experiment to produce one point in the bottom part of Figure~\ref{fig:breakout}.
For MinAtar experiments, we used python version 3.7, Pytorch version 1.4, and public code made available on Github for MinAtar\footnote{https://github.com/kenjyoung/MinAtar}.

\section{Convergence of TDRC} \label{sec: proof_main_thm}
In this section, we prove Theorem \ref{thm:convergence_tdrc}. Our analysis closely follows the one timescale proof for TDC convergence (Maei, 2011). We provide the full proof here for completeness.

\begin{figure*}[!t]
  \begin{tcolorbox}[title=\textbf{Box 1:} Derivation of Eq. \ref{eq: det_G_lambda}., colback=white, colbacktitle=white, coltitle=black, sharp corners, size=small]
    Following the analysis given in Maei (2011), we write
    \begin{equation*}
      \text{det}(\Gvec - \lambda \Ivec) = \text{det} \begin{bmatrix} -\eta \Cvec_\beta - \lambda \Ivec & - \eta \Avec \\  \Avec^\top - \Cvec & -\Avec - \lambda \Ivec \end{bmatrix} = (-1)^{2d} \; \text{det} \begin{bmatrix} \eta \Cvec_\beta + \lambda \Ivec & \eta \Avec \\ \Cvec - \Avec^\top & \Avec + \lambda \Ivec \end{bmatrix}.
    \end{equation*}
    For a matrix ${\bf U} = \begin{bmatrix} \Avec_1 & \Avec_2 \\ \Avec_3 & \Avec_4 \end{bmatrix}$, $\text{det}({\bf U}) = \text{det}(\Avec_1) \cdot \text{det}(\Avec_4 - \Avec_3 \Avec_1^{-1} \Avec_2)$. Further, since $\Cvec$ is positive semi--definite, $\Cvec_\beta + \lambda \Ivec$ would be non--singular for any $\beta > 0$. Using these results, we get
    \begin{equation}
      \tag{B1}
      \text{det}(\Gvec - \lambda \Ivec) = \text{det}(\eta \Cvec + (\eta \beta +\lambda) \Ivec) \cdot \text{det}(\Avec + \lambda \Ivec - \eta (\Cvec - \Avec^\top) (\eta \Cvec + (\eta \beta + \lambda) \Ivec)^{-1} \Avec). \label{eq: det_G}
    \end{equation}
    Now $\eta \Cvec \big(\eta \Cvec + (\eta \beta + \lambda)\Ivec\big)^{-1} = \Big(\big(\eta \Cvec + (\eta \beta + \lambda)\Ivec\big) - (\eta \beta + \lambda)\Ivec \Big) \big(\eta \Cvec + (\eta \beta + \lambda)\Ivec\big)^{-1} = \Ivec - (\eta \beta + \lambda) \big(\eta \Cvec + (\eta \beta + \lambda) \Ivec\big)^{-1}$. We can then write
    \begin{align*}
      & \Avec + \lambda \Ivec - \eta (\Cvec - \Avec^\top) (\eta \Cvec + (\eta \beta + \lambda) \Ivec)^{-1} \Avec \\
      =& \Avec + \lambda \Ivec - \eta \Cvec(\eta \Cvec + (\eta \beta + \lambda) \Ivec)^{-1} \Avec + \eta \Avec^\top (\eta \Cvec + (\eta \beta + \lambda) \Ivec)^{-1} \Avec \\
      =& \Avec + \lambda \Ivec - \bigg( \Ivec - (\eta \beta + \lambda) \big(\eta \Cvec + (\eta \beta + \lambda) \Ivec\big)^{-1}\bigg) \Avec + \eta \Avec^\top (\eta \Cvec + (\eta \beta + \lambda) \Ivec)^{-1} \Avec \\
      =& \lambda \Ivec + (\eta \beta + \lambda) \big(\eta \Cvec + (\eta \beta + \lambda) \Ivec\big)^{-1} \Avec + \eta \Avec^\top (\eta \Cvec + (\eta \beta + \lambda) \Ivec)^{-1} \Avec \\
      =& \Bigg[ \lambda \left(\Avec\right)^{-1} \big(\eta \Cvec + (\eta \beta + \lambda) \Ivec\big) + (\eta \beta + \lambda) \Ivec + \eta \Avec^\top \Bigg]\big(\eta \Cvec + (\eta \beta + \lambda) \Ivec\big)^{-1} \Avec \\
      =& \left(\Avec\right)^{-1} \Bigg[ \lambda \big(\eta \Cvec + (\eta \beta + \lambda) \Ivec\big) + \Avec \big(\eta \Avec^\top + (\eta \beta + \lambda) \Ivec \big) \Bigg] \big(\eta \Cvec + (\eta \beta + \lambda) \Ivec\big)^{-1} \Avec.
    \end{align*}

    Putting the above result in Eq. \ref{eq: det_G} along with the fact that $\text{det}(\Avec_1 \Avec_2) = \text{det}(\Avec_1) \cdot \text{det}(\Avec_2)$, we get
    \begin{equation*}
      \text{det}(\Gvec - \lambda \Ivec) = \det \Big( \lambda \big(\eta \Cvec + (\eta \beta + \lambda) \Ivec\big) + \Avec \big(\eta \Avec^\top + (\eta \beta + \lambda) \Ivec \big)\Big).
    \end{equation*}
  \end{tcolorbox}
\end{figure*}

\subsection{Reformulating the TDRC Update}
We combine the TDRC update equations (Eq. \ref{eq: tdrc_h} and \ref{eq: tdrc_w}) into a single linear system in variable $\varrhovec_t^\top \defeq \begin{bmatrix} \hvec_t^\top \wvec_t^\top \end{bmatrix}$:
\begin{equation}
  \varrhovec_{t+1} = \varrhovec_t + \alpha_t (\Gvec_{t+1} \varrhovec_t + \gvec_{t+1}), \label{eq: tdrc_vec}
\end{equation}
with $\Gvec_{t+1} \defeq \begin{bmatrix} - \eta (\xvec_t \xvec_t^\top + \beta \Ivec) & \eta \rho_t \xvec_t(\gamma\xvec_{t+1} - \xvec_t)^\top \\ - \rho_t (\gamma \xvec_{t+1} \xvec_t^\top) & \rho_t \xvec_t(\gamma \xvec_{t+1} - \xvec_t)^\top \end{bmatrix} $ and $\gvec_{t+1} \defeq \begin{bmatrix} \eta \rho_t R_{t+1} \xvec_t \\ \rho_t R_{t+1} \xvec_t \end{bmatrix} $.

For a random variable $\Xvec$, using the definition of importance sampling,  we know that $\mathbb{E}_b[\rho \Xvec] = \mathbb{E}_\pi[\Xvec]$. Further, while learning off--policy we assume the excursion setting and use the stationary state distribution corresponding to the behavior policy, i.e. $\mathbb{E}_\pi[\xvec_t \xvec_t^\top] = \sum_{S \in \mathcal{S}} d_b(S) \xvec(S) \xvec(S)^\top$, and consequently $\mathbb{E}_b[\xvec_t \xvec_t^\top] = \mathbb{E}_\pi[\xvec_t \xvec_t^\top]$. Therefore, $\Gvec \defeq \mathbb{E}_b[\Gvec_k] = \begin{bmatrix} -\eta \Cvec_\beta & - \eta \Avec \\ \Avec^\top - \Cvec & -\Avec \end{bmatrix}$ and $\gvec \defeq \mathbb{E}_b[\gvec_k] = \begin{bmatrix} \eta \bvec \\ \bvec \end{bmatrix}$, and Eq. \ref{eq: tdrc_vec} can be rewritten as
\begin{equation}
  \varrhovec_{t+1} = \varrhovec_t + \alpha_t \big(h(\varrhovec_t) + M_{t+1}\big),
\end{equation}
where $h(\varrhovec) \defeq \Gvec\varrhovec + \gvec$ and $M_{t+1} \defeq (\Gvec_{t+1} - \Gvec) \varrhovec_t + (\gvec_{t+1} - \gvec)$ is the noise sequence. Also, let $\mathcal{F}_t \defeq \sigma(\varrhovec_1, M_1, \ldots, \varrhovec_{t-1}, M_t)$.

\subsection{Main Proof}
To prove the convergence of TDRC, we use the results from Borkar \& Meyn (2000) which require the following to be true: (i) The function $h(\varrhovec)$ is Lipschitz and there exists $h_\infty(\varrhovec) \defeq \lim_{c \rightarrow \infty} \frac{h(c\varrhovec) }{c}$ for all $\varrhovec \in \mathbb{R}^{2d}$; (ii) The sequence $(M_t, \mathcal{F}_t)$ is a Martingale difference sequence (MDS), and $\E{\|M_{t+1}\|^2 \;|\; \mathcal{F}_t} \leq c_0 (1 + \|\varrhovec\|^2)$ for any initial parameter vector $\varrhovec_1$ and some constant $c_0 > 0$; (iii) The stepsize sequence $\alpha_t$ satisfies $\sum_t \alpha_t = \infty$ and $\sum_t \alpha_t^2 < \infty$; (iv) The origin is a globally asymptotically stable equilibrium for the ODE $\dot{\varrhovec} = h_\infty (\varrhovec)$; and (v) The ODE $\dot{\varrhovec} = h (\varrhovec)$ has a unique globally asymptotically stable equilibrium.

The function $\hvec(\varrhovec) = \Gvec\varrhovec + \gvec$ is Lipschitz with the coefficient $\|\Gvec\|$ and $\hvec_\infty(\varrhovec) = \Gvec\varrhovec$ is well defined for all $\varrhovec \in \mathbb{R}^{2d}$. $(M_{t}, \mathcal{F}_t)$ is an MDS, since by construction it satisfies $\E{M_{t+1} \;|\; \mathcal{F}_t} = 0$ and $M_{t} \in \mathcal{F}_t$. The coverage assumption implies that the second moments of $\rho_t$ are uniformly bounded. Then applying triangle inequality to $M_{t+1} = (\Gvec_{t+1} - \Gvec) \varrhovec_t + (\gvec_{t+1} - \gvec)$ and using the boundedness of second moments of the quadruplets $(\xvec_t, R_t, \xvec_{t+1}, \rho_t)$, we get $\E{\|M_{t+1}\|^2 \;|\; \mathcal{F}_t} \leq \E{\|(\Gvec_{t+1} - \Gvec) \varrhovec_t\|^2 \;|\; \mathcal{F}_t} + \E{\|\gvec_{t+1} - \gvec\|^2 \;|\; \mathcal{F}_t} \leq c_0 (\|\varrhovec_t\|^2 + 1)$. Condition on the stepsizes follows from our assumptions in the theorem statement. To verify the conditions (iv) and (v), we first show that the real parts of all the eigenvalues of $\Gvec$ are negative.

\begin{figure*}[!t]
  \begin{tcolorbox}[title=\textbf{Box 2:} Solutions of Eq. \ref{eq: quadratic}., colback=white, colbacktitle=white, coltitle=black, sharp corners, size=small]
    The solutions of a quadratic $ax^2 + bx + c = 0$ are given by $x = -\frac{b}{2a} \pm \frac{\sqrt{b^2 - 4ac}}{2a}$. Using this, we solve for $\lambda$ in Eq. \ref{eq: quadratic}:
    \begin{align*}
      2 \lambda &= -(\eta \beta + \eta b_c +\lambda_z) \pm \sqrt{(\eta \beta + \eta b_c +\lambda_z)^2 - 4 \eta ( \beta\lambda_z + b_a)}  \\
      &= -\big(\eta \beta + \eta b_c + (\lambda_r + \lambda_c i)\big) \pm \sqrt{\big(\eta \beta + \eta b_c + (\lambda_r + \lambda_c i)\big)^2 - 4 \eta \big(  \beta (\lambda_r + \lambda_c i) + b_a\big)} \\
      &= -\Omega - \lambda_c i \pm \sqrt{(\Omega + \lambda_c i)^2 - 4\eta( \beta \lambda_r + b_a) - 4 \eta \beta \lambda_c i}  \\
      &= -\Omega - \lambda_c i \pm \sqrt{\big(\Omega^2 - \lambda_c^2 - 4 \eta (\beta \lambda_r + b_a)\big) + \big(2\Omega \lambda_c - 4 \eta \beta \lambda_c\big) i}  \\
      &= -\Omega - \lambda_c i \pm \sqrt{\big(\Omega^2 - \Xi \big) + \big(2\Omega \lambda_c - 4 \eta \beta \lambda_c\big) i},
    \end{align*}

    where in the second step we put $\lambda_z = \lambda_r + \lambda_c i$, and also we define $\Omega = \eta \beta + \eta b_c + \lambda_r$ and $\Xi = \lambda_c^2 + 4\eta( \beta \lambda_r + b_a)$, which are both real numbers.
    %%  \end{thmbox}
  \end{tcolorbox}
\end{figure*}

\subsection{Proving that the Real Parts of Eigenvalues of $\Gvec$ are Negative (assuming $\Cvec$ to be non--Singular)} \label{sec: prove_G}
In this section, we consider the case when the $\Cvec$ matrix is non--singular. TDRC converges even when $\Cvec$ is singular under alternate conditions, which are given in Section \ref{sec: convergence_proof_C_singular}. From Box 1, we obtain
\begin{align}
  \text{det}(\Gvec - \lambda \Ivec) =& \text{det}\Big(\lambda (\eta \Cvec + (\eta \beta + \lambda) \Ivec) \nonumber \\
  & \quad \quad + \Avec (\eta \Avec^\top + (\eta \beta + \lambda) \Ivec )\Big), \label{eq: det_G_lambda}
\end{align}
for some $\lambda \in \mathbb{C}$. Now because an eigenvalue $\lambda$ of matrix $\Gvec$ satisfies $\text{det}(\Gvec - \lambda \Ivec) = 0$, there must exist a non--zero vector $\zvec \in \mathbb{C}^d$ such that $\zvec^* [\lambda (\eta \Cvec + (\eta \beta + \lambda) \Ivec) + \Avec (\eta \Avec^\top + (\eta \beta + \lambda) \Ivec )] \zvec = 0$, which is equivalent to
\begin{align*}
  \lambda^2 + \left(\eta \beta + \eta \frac{\zvec^* \Cvec \zvec}{\|\zvec\|^2} + \frac{{\zvec^*\Avec\zvec}}{\|\zvec\|^2} \right) \lambda \quad \quad \quad \quad & \\
  \quad \quad \quad \quad + \eta \left( \beta \frac{{\zvec^*\Avec\zvec}}{\|\zvec\|^2} +  \frac{\zvec^*\Avec \Avec^\top \zvec}{\|\zvec\|^2} \right) &= 0.
\end{align*}
We define $b_c = \frac{\zvec^* \Cvec \zvec}{\|\zvec\|^2}$, $b_a = \frac{\zvec^*\Avec \Avec^\top \zvec}{\|\zvec\|^2}$, and ${\lambda_z} = \frac{{\zvec^*\Avec\zvec}}{\|\zvec\|^2} \equiv \lambda_r + \lambda_c{i}$ for $\lambda_r, \lambda_c \in \mathbb{R}$. The constants $b_c$ and $b_a$ are real and greater than zero for all non--zero vectors $\zvec$. Then the above equation can be written as
\begin{equation}
  \lambda^2 + \left(\eta \beta + \eta b_c + {\lambda_z} \right) \lambda + \eta( \beta {\lambda_z} + b_a ) = 0. \label{eq: quadratic}
\end{equation}

We solve for $\lambda$ in Eq. \ref{eq: quadratic} (see Box 2 for the full derivation) to obtain $2 \lambda = -\Omega - \lambda_c {i} \pm \sqrt{(\Omega^2 - \Xi ) + (2\Omega \lambda_c - 4 \eta \beta \lambda_c) {i}}$, where we introduced intermediate variables $\Omega = \eta \beta + \eta b_c + \lambda_r$, and $\Xi = \lambda_c^2 + 4 \eta( \beta \lambda_r + b_a)$, which are both real numbers.

Using $\text{Re}(\sqrt{x + y {i}}) = \pm \frac{1}{\sqrt{2}} \sqrt{\sqrt{x^2 + y^2} + x}$ we get $\text{Re}(2 \lambda) = -\Omega \pm \frac{1}{\sqrt{2}} \sqrt{\Upsilon}$, with the intermediate variable $\Upsilon = \sqrt{(\Omega^2 - \Xi)^2 + (2\Omega \lambda_c - 4 \eta \beta \lambda_c)^2} + (\Omega^2 - \Xi)$. Next we obtain conditions on $\beta$ and $\eta$ such that the real parts of both the values of $\lambda$ are negative for all non--zero vectors $\zvec \in \mathbb{C}$.

\begin{figure*}
  \begin{tcolorbox}[title=\textbf{Box 3:} Simplification of Eq. \ref{eq: omega_inequality}., colback=white, colbacktitle=white, coltitle=black, sharp corners, size=small]
    Putting the value of $\Upsilon = \sqrt{(\Omega^2 - \Xi)^2 + (2\Omega \lambda_c - 4 \eta \beta \lambda_c)^2} + (\Omega^2 - \Xi)$ back in $\Omega > \frac{1}{\sqrt{2}} \sqrt{\Upsilon}$, we get
    \begin{align*}
      && \Omega &> \frac{1}{\sqrt{2}} \sqrt{\sqrt{(\Omega^2 - \Xi)^2 + (2\Omega \lambda_c - 4 \eta \beta \lambda_c)^2} + (\Omega^2 - \Xi)} \\
      \Leftrightarrow \quad && \Omega^2 &> \frac{1}{2} \left[\sqrt{(\Omega^2 - \Xi)^2 + (2\Omega \lambda_c - 4 \eta \beta \lambda_c)^2} + (\Omega^2 - \Xi)\right] \tag*{[squaring both sides]} \\
      \Leftrightarrow \quad && \Omega^2 + \Xi &> \sqrt{(\Omega^2 - \Xi)^2 + (2\Omega \lambda_c - 4 \eta \beta \lambda_c)^2} \\
      \Leftrightarrow \quad && (\Omega^2 + \Xi)^2 &> (\Omega^2 - \Xi)^2 + (2\Omega \lambda_c - 4 \eta \beta \lambda_c)^2 \tag*{[squaring both sides]} \\
      \Leftrightarrow \quad && \Omega^2 \Xi &> (\Omega \lambda_c - 2 \eta \beta \lambda_c)^2 \\
      \Leftrightarrow \quad && \Omega^2 (\lambda_c^2 + 4 \eta( \beta \lambda_r +  b_a)) &> \Omega^2 \lambda_c^2 + 4 \eta^2 \beta^2 \lambda_c^2 - 4 \eta \beta \lambda_c^2 \Omega \tag*{[putting $\Xi = \lambda_c^2 + 4 \eta (\beta \lambda_r + b_a)$]}\\
      \Leftrightarrow \quad && \Omega^2 \eta ( \beta \lambda_r + b_a) &> \eta^2 \beta^2 \lambda_c^2 - \eta \beta \lambda_c^2 \Omega \\
      \Leftrightarrow \quad && (\eta\beta + \eta b_c + \lambda_r)^2 (\beta \lambda_r + b_a) &> \eta \beta^2 \lambda_c^2 - \beta \lambda_c^2 (\eta \beta + \eta b_c + \lambda_r) \tag*{[putting $\Omega = \eta \beta + \eta b_c + \lambda_r$]} \\
      \Leftrightarrow \quad && (\eta\beta + \eta b_c + \lambda_r)^2 (\beta \lambda_r + b_a) &> - \beta \lambda_c^2 (\eta b_c + \lambda_r) \\ \\
      \Leftrightarrow \quad &&& (\eta\beta + \eta b_c + \lambda_r)^2 (\beta \lambda_r + b_a) + \beta \lambda_c^2 (\eta b_c + \lambda_r) > 0.
    \end{align*}

    Note that all these steps have full equivalence (especially the squaring operations in second and fourth step are completely reversible), because we explicitly enforce that $\Omega > 0$ and $\Omega^2 + \Xi > 0$ in Conditions \ref{eq: first_cond} and \ref{eq: second_cond} respectively. As a result, if we satisfy conditions \ref{eq: first_cond}, \ref{eq: second_cond}, and \ref{eq: third_cond}, $\text{Re}(2 \lambda) = -\Omega + \frac{1}{\sqrt{2}} \sqrt{\Upsilon} < 0$ would be satisfied as well.
  \end{tcolorbox}
\end{figure*}

\subsubsection{CASE 1}
First consider $\text{Re}(2 \lambda) = -\Omega + \frac{1}{\sqrt{2}} \sqrt{\Upsilon}$. Then $\text{Re}(\lambda) < 0$ is equivalent to
\begin{equation}
  \Omega > \frac{1}{\sqrt{2}} \sqrt{\Upsilon}. \label{eq: omega_inequality}
\end{equation}

  Since, the right hand side of this inequality is clearly positive, we must have
\begin{equation}
  \tag{C1}
  \Omega = \eta \beta + \eta b_c + \lambda_r > 0. \label{eq: first_cond}
\end{equation}

This gives us our first condition on $\eta$ and $\beta$. Simplifying Eq. \ref{eq: omega_inequality} and putting back the values for the intermediate variables (see Box 3 for details), we get
\begin{equation}
  \Omega^2 + \Xi > \sqrt{(\Omega^2 - \Xi)^2 + (2\Omega \lambda_c - 4 \eta \beta \lambda_c)^2}. \label{eq: omega_xi}
\end{equation}
Again, since the right hand side of the above inequality is positive, we must have
\begin{equation}
  \tag{C2}
  \Omega^2 + \Xi = (\eta \beta + \eta b_c + \lambda_r)^2 + \lambda_c^2 + 4 \eta(\beta \lambda_r + b_a) > 0. \label{eq: second_cond}
\end{equation}
This is the second condition we have on $\eta$ and $\beta$. Continuing to simplify the inequality in Eq. \ref{eq: omega_xi} (again see Box 3 for details), we get our third and final condition:
\begin{equation}
  \tag{C3}
(\eta\beta + \eta b_c + \lambda_r)^2 (\beta \lambda_r + b_a) + \beta \lambda_c^2 (\eta b_c + \lambda_r) > 0. \label{eq: third_cond}
\end{equation}
If $\lambda_r > 0$ for all $\zvec \in \mathbb{R}$, then each of the Conditions \ref{eq: first_cond}, \ref{eq: second_cond}, and \ref{eq: third_cond} hold true and consequently TDRC converges. This case corresponds to the on--policy setting where the matrix $\Avec$ is positive definite and TD converges.

Now we show that TDRC converges even when $\Avec$ is not PSD (the case where TD is not guaranteed to converge). If we assume $\beta \lambda_r + b_a > 0$ and $\eta b_c + \lambda_r > 0$, then each of the Conditions \ref{eq: first_cond}, \ref{eq: second_cond}, and \ref{eq: third_cond} again hold true and TDRC would converge. As a result we obtain the following bounds:
\begin{align}
  \beta &< - \frac{b_a}{\lambda_r} \Rightarrow \beta < \min_{\zvec} \left(- \frac{\zvec^*\Avec \Avec^\top \zvec}{\zvec^*\Hvec\zvec}\right), \label{eq: beta_cond_old}\\
  \eta &> - \frac{\lambda_r}{b_c} \Rightarrow \eta > \max_{\zvec} \left(- \frac{\zvec^*\Hvec\zvec}{\zvec^* \Cvec \zvec}\right), \label{eq: eta_cond_old}
\end{align}
with $\Hvec \defeq \frac{1}{2}(\Avec + \Avec^\top)$. These bounds can be made more interpretable. Using the substitution $\yvec = \Hvec^{\frac{1}{2}} \zvec$ we obtain
\begin{align*}
  \min_{\zvec} \left(- \frac{\zvec^*\Avec \Avec^\top \zvec}{\zvec^*\Hvec\zvec}\right) &\equiv \min_{\yvec} \frac{\yvec^* (-\Hvec^{-\frac{1}{2}} \Avec \Avec^\top \Hvec^{-\frac{1}{2}}) \yvec}{\|\yvec\|^2} \\
  &= \lambda_{\text{min}} (-\Hvec^{-\frac{1}{2}} \Avec \Avec^\top \Hvec^{-\frac{1}{2}}) \\
  &= - \lambda_{\text{max}} (\Hvec^{-\frac{1}{2}} \Avec \Avec^\top \Hvec^{-\frac{1}{2}}) \\
  &= - \lambda_{\text{max}} (\Hvec^{-1} \Avec \Avec^\top),
\end{align*}
where $\lambda_{\text{max}}$ represents the maximum eigenvalue of the matrix. Proceeding similarly for $\eta$, we can write the bounds in Eq. \ref{eq: beta_cond_old} and \ref{eq: eta_cond_old} equivalently as
\begin{align}
  \beta &< - \lambda_{\text{max}} (\Hvec^{-1} \Avec \Avec^\top), \label{eq: beta_cond} \\
  \eta &> - \lambda_{\text{min}} (\Cvec^{-1} \Hvec). \label{eq: eta_cond}
\end{align}
If these bounds are satisfied by $\eta$ and $\beta$ then the real parts of all the eigenvalues of $\Gvec$ would be negative and TDRC will converge.

\subsubsection{CASE 2} Next consider $\text{Re}(2 \lambda) = -\Omega - \frac{1}{\sqrt{2}} \sqrt{\Upsilon}$. The second term is always negative and we assumed $\Omega > 0$ in Eq. \ref{eq: first_cond}. As a result, $\text{Re}(\lambda) < 0$ and we are done.

Therefore, we get that the real part of the eigenvalues of $\Gvec$ are negative and consequently condition (iv) above is satisfied. To show that condition (v) holds true, note that since we assumed $\Avec + \beta \Ivec$ to be non--singular, $\Gvec$ is also non--singular; this means that for the ODE $\dot{\varrhovec} = h (\varrhovec)$, $\varrhovec^* = - \Gvec^{-1} \gvec$ is the unique asymptotically stable equilibrium with $\bar{\Vvec}(\varrhovec) \defeq \frac{1}{2} (\Gvec \varrhovec + \gvec)^\top (\Gvec \varrhovec + \gvec)$ as its associated strict Lyapunov function.

\subsection{Convergence of TDRC when $\Cvec$ is Singular} \label{sec: convergence_proof_C_singular}
When $\Cvec$ is singular, $b_c = \frac{\zvec^* \Cvec \zvec}{\|\zvec\|^2}$ is no longer always greater than zero for an arbitrary vector $\zvec$. Consequently, if we explicitly set $b_c = 0$ we would get alternative bounds on $\eta$ and $\beta$ for which TDRC would converge. Putting $b_c = 0$ in Conditions \ref{eq: first_cond}, \ref{eq: second_cond}, and \ref{eq: third_cond}, we get
\begin{align*}
  \eta \beta + \lambda_r &> 0, \\
  (\eta \beta + \lambda_r)^2 + \lambda_c^2 + 4 \eta(\beta \lambda_r + b_a) &> 0, \text{and} \\
  (\eta\beta + \lambda_r)^2 (\beta \lambda_r + b_a) + \beta \lambda_c^2 \lambda_r &> 0.
\end{align*}
As before, we are concerned with the case when $\Avec$ is not PSD and thus $\lambda_r < 0$. Further, assume that $\beta \lambda_r + b_a > 0$ (this is the same upper bound on $\beta$ as given in Eq. \ref{eq: beta_cond_old}). We simplify the third inequality above to obtain the bound on $\eta$. As a result, we get the following bounds for $\beta$ and $\eta$:
\begin{align}
  \beta < -\frac{b_a}{\lambda_r}, \quad \quad \quad \eta > \frac{1}{\beta} \left(\sqrt{\frac{-\beta \lambda_c^2 \lambda_r}{\beta \lambda_r + b_a}} - \lambda_r \right). \label{eq: alternate_cond_tdrc_convergence_C_singular}
\end{align}
The bound on $\eta$ automatically satisfies the first condition $\eta \beta + \lambda_r > 0$. Therefore, if $\beta$ and $\eta$ satisfy these bounds, TDRC converges even for a singular $\Cvec$ matrix.

\section{Fixed Points of TDRC} \label{sec: proof_fixed_point}
\begin{thm}[Fixed Points of TDRC] \label{thm:fixed_point}
  If $\vecw$ is a TD fixed point, i.e., a solution to $\Amat \vecw =\vecb$, then it is a fixed point for the expected TDRC update,
  \begin{equation*}
    \Amat_\beta^\top \Cmat_\beta^\inv (\vecb -\Amat \vecw)=\zerovec
    .
  \end{equation*}
  Further, the set of fixed points for TD and TDRC are equivalent if $\Cmat_\beta$ is invertible and if
 % either of the following are satisfied:
%  (i) $\Amat$ is positive semi-definite and $\beta > 0$ or
%
 % (ii)
  $-\beta$ does not equal to any of the eigenvalues of $\Amat$.
  Note that $\Cmat_\beta$ is always invertible if $\beta > 0$, and is invertible if $\Cmat$ is invertible even for $\beta = 0$.
\end{thm}
\begin{proof}
To show equivalence, the first part is straightforward: when $\Amat \vecw =\vecb$, then $\vecb -\Amat \vecw = \zerovec$ and so $\Amat_\beta^\top \Cmat_\beta^\inv (\vecb -\Amat \vecw)=\zerovec$.
  This means that any TD fixed point is a TDRC fixed point. Now we simply need to show that under the additional conditions, a TDRC fixed point is a TD fixed point.

  %%==================================================
  %% No longer needed since we don't mention this case in the theorem.
  %%==================================================
  %% (i) If $\Amat$ is positive semi-definite and $\beta >0$, then $-\beta$ cannot equal any of the eigenvalues of $\Amat$, so the second criterion covers this case.
  %% We included this separately, simply to highlight an important case when TD would be used, namely when $\Amat$ is PSD.

  If $-\beta$ does not equal any of the eigenvalues of $\Amat$, then $\Amat_\beta = \Amat + \beta \eye$ is a full rank matrix. Because both $\Amat_\beta$ and $\Cmat_\beta$ are full rank, the nullspace of $\Amat_\beta^\top \Cmat_\beta^\inv (\vecb -\Amat \vecw)$ equals to the nullspace of $\vecb -\Amat \vecw$. Therefore, $\vecw$ satisfies $\Amat_\beta^\top \Cmat_\beta^\inv (\vecb -\Amat \vecw) = 0$ iff $(\vecb -\Amat \vecw) = \zerovec$.

  We can prove Theorem \ref{thm:fixed_point}, in an alternate fashion as well. The linear system in Eq. \ref{eq: tdrc_vec} has a solution (in expectation) which satisfies
\begin{equation*}
  \Gvec \varrhovec + \gvec = \zerovec.
\end{equation*}
We show that this linear system has full rank and thus a single solution: $\wvec = \Avec^{-1} \bvec$ and $\hvec = \zerovec$. If we show that the matrix $\Gvec$ is non--singular, i.e. its determinant is non--zero, we are done. From Eq. \ref{eq: det_G_lambda} it is straightforward to obtain
\begin{equation*}
  \text{det}(\Gvec) = \eta^{2d} \det(\Avec^\top + \beta \Ivec) \cdot \det(\Avec),
\end{equation*}
which is non--zero if we assume that $\beta$ does not equal the negative of any eigenvalue of $\Avec$ and that $\Avec$ is non--singular.
\end{proof}

\end{document}